\documentclass{article}

\PassOptionsToPackage{dvipsnames}{xcolor}

\makeatletter
\def\input@path{{ArXiv/}}
\makeatother
\usepackage{ppe_arxiv,times}

\usepackage[utf8]{inputenc}
\usepackage[T1]{fontenc}
\usepackage{newunicodechar}
\newunicodechar{–}{\textendash}
\makeatletter
\@ifundefined{IfDocumentMetadataT}{\newcommand\IfDocumentMetadataT[1]{}}{}
\makeatother
\usepackage{hyperref}
\usepackage{url}
\usepackage{booktabs}
\usepackage{amsfonts}
\usepackage{amsmath}
\usepackage{amssymb}
\usepackage{amsthm}
\newtheorem{theorem}{Theorem}

\theoremstyle{remark}

\theoremstyle{plain}

\usepackage{microtype}
\usepackage{graphicx}
\usepackage[dvipsnames]{xcolor}
\usepackage[normalem]{ulem}
\usepackage{colortbl}
\usepackage{subcaption}
\usepackage{nicefrac}
\usepackage{pdflscape}
\usepackage{multirow}
\usepackage{makecell}
\usepackage{caption}
\usepackage{placeins}
\usepackage{xspace}
\usepackage{threeparttable}
\usepackage{xcolor}
\usepackage{pifont}
\usepackage{enumitem}

\ifdefined\pdfsuppressptexinfo
\fi
\newcommand{\ours}{\texttt{PPE}\xspace}
\newcommand{\mst}{\textsc{MST}\xspace}
\newcommand{\persformer}{\textsc{Persformer}\xspace}
\newcommand{\xpert}{\textsc{xPerT}\xspace}
\newcommand{\perslay}{\textsc{PersLay}\xspace}
\newcommand{\persimage}{\textsc{PI}\xspace}

\newcommand{\Tref}[1]{Tab.~\ref{#1}}
\newcommand{\Aref}[1]{Appendix~\ref{#1}}

\title{Persistence Positional Encoding: A Minimal Transformer Recipe for Persistence Diagrams}

\author{%
  Eunwoo Heo$^1$\thanks{Equal contribution.} \quad
  Ganghyun Kim$^1$\footnotemark[1] \quad
  Jaejun Yoo$^1$ \\
  $^1$Graduate School of Artificial Intelligence,
  Ulsan National Institute of Science and Technology
}

\graphicspath{{ArXiv/}{./}}
\hypersetup{
  pdftitle={Persistence Positional Encoding: A Minimal Transformer Recipe for Persistence Diagrams},
  pdfauthor={Eunwoo Heo, Ganghyun Kim, Jaejun Yoo}
}

\begin{document}

\maketitle

\begin{abstract}
Persistence diagrams (PDs) represent multiscale topological information as multisets of birth--death pairs $(b,d)$.
How best to use them with Transformers, however, remains unclear. Connecting PDs to Transformers requires choosing how topological information is tokenized and exposed to self-attention. Prior methods use specialized designs for tokenization, aggregation, Transformer blocks, or attention, but are hard to compare since reported results rarely share PD inputs, splits, or model-selection rules.
We instead confine the diagram-specific design to point representation, mapping each birth--death pair through fixed random Fourier features while leaving the Transformer backbone unchanged, and call this recipe Persistence Positional Encoding (PPE). 
In matched controls on eight TU graph benchmarks (TU8), changing only point representation shifts macro accuracy by 5.49 to 6.47 percentage points (pp), against 0.47 to 1.45~pp from changing the downstream module, identifying point representation as the dominant design variable tested. 
We further reevaluate vectorizations, kernels, learned representations, and Transformer-based models under one protocol with shared PD inputs, outer folds, and validation-only selection.
Among all evaluated PD methods and downstream configurations, PPE achieves the highest mean on most benchmarks, including TU8, ORBIT5K, and MCB-C.
Finally, we prove that grid averaging leaves a reconstruction error at any resolution,
whereas a Transformer with fixed random Fourier frequencies is almost surely universal.
PPE thus offers a minimal, Transformer-native baseline for persistence diagrams, and a common protocol for future comparison.
\end{abstract}

\section{Introduction}
\label{sec:intro}
How should a persistence diagram be presented to a Transformer? Persistent homology (PH) is a standard tool in topological data analysis (TDA)~\citep{carlsson2009topology} for characterizing the shape and structure of data across multiple scales, with applications spanning molecular property prediction, protein structure analysis, drug discovery, dynamical systems, materials science, and neuroscience (see Appendix~\ref{app:related} for related work). Its output is commonly represented as a persistence diagram (PD), a multiset of birth--death pairs that records when topological features appear and disappear.
A broad range of methods incorporate PDs into learning pipelines, including
fixed summaries---persistence images, landscapes, Betti curves, silhouettes~\citep{adams2017persistence,bubenik2015statistical,umeda2017betti,chazal2015silhouette};
kernels---sliced Wasserstein, scale-space, persistence-weighted Gaussian~\citep{carriere2017sliced,reininghaus2015stable,kusano2016pwg};
and learned representations---\perslay{}, RationalHat, ATOL~\citep{carriere2020perslay,hofer2019learning,royer2021atol}.
Yet for Transformers, which have become general-purpose backbones across modalities, the question above remains open.

The difficulty is that a Transformer does not consume a diagram. It consumes tokens. 
Using a diagram-level representation as a single token compresses the entire diagram before self-attention is applied, so attention never operates on individual topological features directly. Splitting such a representation into multiple tokens requires an additional choice of how those tokens should correspond to the underlying diagram. Treating each persistence point $(b,d)$ as a token is the natural point-level alternative, but raw two-dimensional coordinates are not necessarily an effective input for a Transformer. Using PDs with Transformers therefore involves a nontrivial choice of both tokenization and point representation.

Existing Transformer-based PD methods place their diagram-specific design at different stages.~\persformer{}~\citep{reinauer2021persformer} keeps standard scaled dot-product attention while adding a learned pointwise embedding, extra residual connections between self-attention layers, and attention pooling against a trainable query.~\xpert{}~\citep{kim2024xpert} pixelizes each diagram onto an $R\times R$ grid and partitions it into $P\times P$ patches, so the grid resolution and the patch size must be chosen together before a ViT-style encoder sees the tokens. The Multiset Transformer (\mst{})~\citep{multiset-transformer-2025} represents distinct points with explicit multiplicities, adds multiplicity-dependent terms to its attention weights, and uses multiset-specific equivariant blocks and invariant readout. Because each of these changes several parts of the representation, aggregation, or Transformer pipeline at once, a method-level comparison cannot tell which change is responsible for the reported gains (see Appendix~\ref{app:related}).

A second obstacle compounds the first. Reported results across the PD literature come from different PD inputs, data splits, model-selection procedures, and search budgets, so numbers from different papers are rarely comparable. This lack of comparability has been raised as a barrier to progress in topological machine learning \citep{barnes2021comparative,castillarieso2026standards}. As a concrete example, the released~\xpert{} protocol trains on all non-test samples and reports the maximum test accuracy across epochs. Under such conditions, even the basic question of whether a specialized architecture is needed at all cannot be settled from the literature.

We address both obstacles with a deliberately minimal design and a controlled evaluation. \emph{Persistence Positional Encoding} (\ours{}) confines its diagram-specific design to token construction. Each birth--death coordinate $(b,d)$ is mapped through fixed random Fourier features and combined with a learnable homology-dimension embedding, while the Transformer blocks, the attention rule, and the CLS readout stay unchanged. 
This separation lets us ask directly whether the bottleneck in applying Transformer to PDs lies in the representation of persistence points or in the architecture around them. In matched controls on the eight TU graph benchmarks, changing only the point representation shifts macro accuracy by 5.49 to 6.47 percentage points (pp), whereas changing only the downstream module shifts it by 0.47 to 1.45~pp (Table~\ref{tab:full_interface_controls}). The largest single-dataset gain from the point encoder is 19.5~pp, on DHFR. Among the components varied here, point representation is the dominant design variable.

Fourier feature mappings expose multiple frequency scales before the coordinates reach the downstream network~\citep{tancik2020fourier,mildenhall2020nerf}, and this gives the empirical finding a theoretical counterpart. We show that at every resolution, uniform grid averaging leaves a reconstruction error that no decoder can bring below a fixed positive bound (Theorem~\ref{thm:grid_capacity}). Information discarded at the grid cannot be recovered downstream, however expressive the model. A Transformer with fixed Fourier frequencies is instead almost surely a universal approximator for continuous functions of persistence diagrams, without prior knowledge of the target frequencies (Theorem~\ref{thm:ppe_fixed_encoder_universality}).

We reevaluate diagram-level vectorizations, kernel methods, learned PD representations, and Transformer-based PD models above on shared prepared PD inputs and outer folds, with model selection based only on validation data for every method. 
Under this protocol,~\ours{} achieves the highest TU8 macro across all evaluated method--downstream configurations and has the highest mean among the evaluated Transformer-based PD methods on all eleven benchmarks, including 90.68\% on ORBIT5K, 6.82~pp above the next-best 83.86\%.
This includes two larger-scale benchmarks not reported in prior PD--Transformer work, with 69.32 ROC-AUC on ogbg-molhiv (41{,}127 graphs under its official scaffold split) and 65.79 balanced accuracy on MCB-C (1{,}745 point-cloud samples), exceeding the next-best Transformer-based method by 1.05 and 9.46~pp, respectively.

One may ask whether PPE is more than Fourier features applied to PD coordinates. 
The feature map itself follows a standard construction~\citep{tancik2020fourier}.
Our contribution is not the feature map in isolation, but showing why this representation choice matters through the matched controls that attribute the gains to point representation rather than to architecture, the theory that explains why grid-based tokenization cannot close the gap, and the unified protocol under which Transformer-based, vectorization, kernel, and learned PD methods are compared on equal terms. 
Together, these turn a small encoder change into something the field can build on.~\ours{} gives persistence diagrams a Transformer-native interface that requires no PD-specific attention or pooling, so it drops into standard token-based pipelines as they are. The controls and theory tell practitioners which design variable deserves their effort, and the unified protocol offers a common ground on which future PD methods, Transformer-based or not, can be compared. We view this combination, rather than the encoder alone, as the contribution of this work.
Our main contributions are as follows:

\begin{itemize}[label={},leftmargin=1.2em]
\item \textbf{Method and findings.}
We introduce~\ours{}, which maps PD points to Fourier features for a standard Transformer. It achieves the highest mean among the evaluated Transformer-based PD methods on all eleven benchmarks, including two larger-scale evaluations, and shows more stable convergence (Figure~\ref{fig:orbit5k_encoder_convergence}). Matched TU8 controls show larger gains from changing the point encoding, 5.49--6.47~pp, than from changing the downstream module, 0.47--1.45~pp (Table~\ref{tab:full_interface_controls}).
\item \textbf{Theory.}
We prove a reconstruction lower bound for uniform grid averaging that holds at every resolution and for every decoder (Theorem~\ref{thm:grid_capacity}), and an almost-sure universal approximation result for the~\ours{} Transformer with Fourier frequencies sampled once and kept fixed (Theorem~\ref{thm:ppe_fixed_encoder_universality}).
\item \textbf{Unified evaluation.}
We reevaluate Transformer-based models, vectorization methods, kernels, and learned PD representations on shared PD inputs and outer folds with validation-based model selection (Tables~\ref{tab:graph_benchmarks} and~\ref{tab:tu8_method_summary}). We further extend the comparison beyond the benchmarks reported in prior PD--Transformer work to the 41{,}127-graph ogbg-molhiv and MCB-C.~\ours{} attains the highest TU8 macro accuracy among the evaluated method--downstream configurations (Table~\ref{tab:tu8_method_summary}).
\end{itemize}

\section{Method}
\label{sec:method}

\subsection{Preliminary}

A~\emph{persistence diagram} $D$ is a multiset of points $(b, d) \in \mathbb{R}^2$ where each point represents a topological feature with birth time $b$ and death time $d$ (with $b \leq d$). We write $|D|$ for its cardinality. In practice, we compute persistence diagrams for multiple homological dimensions.
Let $D^{(k)}$ denote the persistence diagram for dimension $k$ (Betti-$k$ features). In our experiments, we use $D^{(0)}$ for connected components (clusters) and $D^{(1)}$ for loops/cycles (holes). The model consumes the points of all homology dimensions used as a single sequence, so we index them jointly by $i = 1, \ldots, N$ with $N = \sum_k |D^{(k)}|$, writing $(b_i, d_i)$ for the $i$-th point and $k_i \in \{0, 1, \ldots\}$ for its homology dimension.

\subsection{Persistence Positional Encoding (\ours{})}

\ours{} treats a PD as an unordered token sequence rather than rasterizing it into an image or reducing it to pairwise kernel evaluations.
For each diagram point, we encode the raw $(b,d)$ coordinate with fixed random Fourier features, project the result to the Transformer width, and add a learnable homology-dimension embedding that marks whether the token comes from $H_0$, $H_1$, or another dimension.
The resulting variable-length token sequence is processed by an otherwise standard Transformer encoder with padding masks, CLS-token pooling, and a linear prediction head.
Figure~\ref{fig:architecture} summarizes this pipeline. The following subsections define each component.

\begin{figure}[t]
    \centering
    \includegraphics[width=\linewidth]{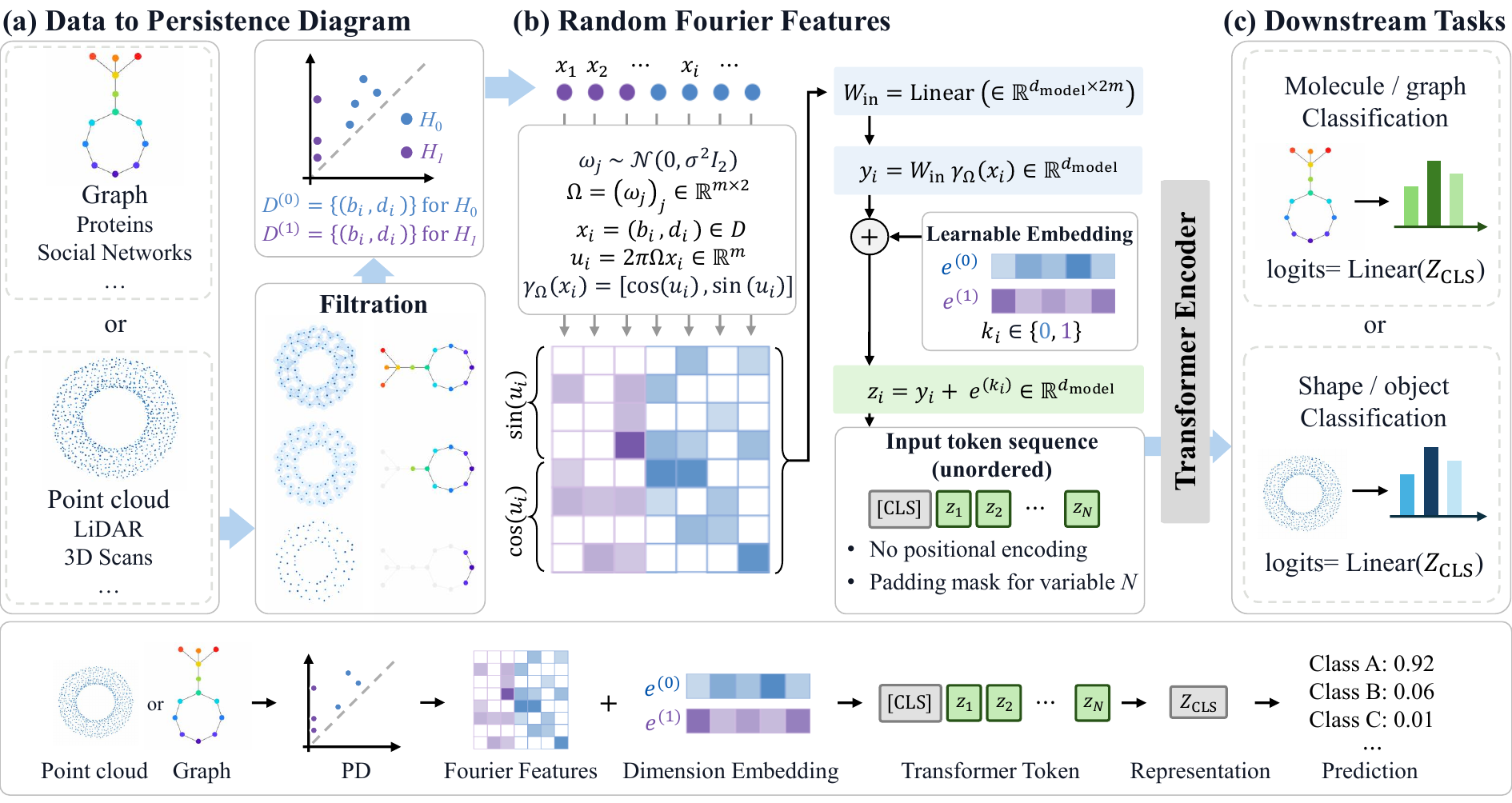}
    \caption{
    \textbf{A schematic overview of \ours{}.}
    \textbf{(a)} Filtration converts topological data, e.g., graphs, point clouds, or other structured domains, into persistence diagrams. 
    \textbf{(b)} Each PD point $x_i=(b_i,d_i)$ is encoded by fixed random Fourier features, linearly projected, and augmented with a learnable homology-dimension embedding ($H_0$ or $H_1$). The resulting unordered token sequence is processed by a Transformer encoder without positional encoding. 
    \textbf{(c)} The CLS representation drives downstream tasks, such as molecular/graph classification and shape/object classification. 
    }
    \label{fig:architecture}
\end{figure}

\subsubsection{Fourier Feature Encoding}

Following~\citet{tancik2020fourier}, we encode each coordinate $(b, d) \in \mathbb{R}^2$ using random Fourier features (RFF).
Let $\Omega \in \mathbb{R}^{m \times 2}$ have rows $\omega_1^\top, \dots, \omega_m^\top$, where the frequencies $\omega_j \in \mathbb{R}^2$ are sampled from $\mathcal{N}(0, \sigma^2 I_2)$ and fixed throughout training.
The RFF encoding maps a 2D coordinate to a $2m$-dimensional vector:
\begin{equation}
    \gamma_\Omega(x) =
    \bigl[
        \cos(2\pi \Omega x);\;
        \sin(2\pi \Omega x)
    \bigr] \in \mathbb{R}^{2m}
    \label{eq:rff}
\end{equation}
where $[\cdot\,;\,\cdot]$ concatenates the $m$-dimensional cosine and sine blocks.
This encoding keeps each PD point as a separate token and provides explicit multi-frequency features to mitigate the spectral bias that can hinder raw-coordinate MLPs from resolving fine-grained differences (Section~\ref{sec:encoder_ablation}).

\subsubsection{Homology-Dimension Embedding}

Each point $(b_i, d_i)$ is mapped to a token by linearly projecting its Fourier features to the model dimension $d_{\mathrm{model}}$ and adding a learnable embedding of its homology dimension $k_i$, with one embedding vector $e^{(k)}$ per homology dimension:
\begin{equation}
    z_i = W_{\text{in}}\,\gamma_\Omega(b_i, d_i) + e^{(k_i)},
    \qquad W_{\text{in}} \in \mathbb{R}^{d_{\mathrm{model}} \times 2m},\ e^{(k)} \in \mathbb{R}^{d_{\mathrm{model}}}.
\end{equation}

\subsubsection{Encoder, Head, and Training}

The token sequence $[z_1, \ldots, z_N]$ is prepended with a learnable [CLS] token~\citep{devlin2019bert} and passed through $L$ vanilla Transformer encoder layers~\citep{vaswani2017attention} with pre-LayerNorm~\citep{xiong2020layer} and GELU activations. Variable-length batches are padded to the maximum length per batch, and an attention mask zeroes out padded positions so they do not influence attention or downstream representations. Classification uses a single linear head over the LayerNorm-ed CLS output. We train with cross-entropy loss and AdamW~\citep{loshchilov2019decoupled} under a linear-warmup-then-cosine-decay schedule with early stopping on validation accuracy. Full hyperparameters appear in Appendix~\ref{app:architecture}.

\section{Theory}
\label{sec:theory}

We establish two complementary expressivity results. First, we show that a cell-average grid representation has an irreducible worst-case reconstruction error, regardless of decoder complexity. Second, we show that fixing the Fourier frequencies does not create an analogous limitation for \ours{}. With probability one, a two-layer Transformer with a fixed set of randomly sampled Fourier frequencies can universally approximate continuous functions on bounded persistence diagrams.
Here, \(C(X)\) denotes the space of real-valued continuous functions on compact \(X\subset\mathbb R^2\). 
Proofs of both theorems are provided in Appendix~\ref{app:Expressivity_and_Approximation}.

\subsection{Limitation of grid-based representations}
\label{sec:theory:grid}
By the Stone--Weierstrass theorem~\citep{stone1948generalized}, for any continuous function $f$ on a compact domain $X$ and any desired accuracy, there exists a finite trigonometric sum that achieves that accuracy. 
That is, choosing frequencies according to the target function yields arbitrarily accurate approximations of continuous functions. 
However, when a continuous function $f$ is represented by finite-resolution cell statistics, distinct within-cell variations can become indistinguishable before any downstream processing. The following theorem formalizes this limitation for a cell-average grid representation.

\begin{theorem}[Grid reconstruction lower bound]
\label{thm:grid_capacity}
Let $X=[0,c]^2$, and let $\{C_{ij}\}_{i,j=1}^R$ denote the uniform $R\times R$ partition of $X$. For any $f\in C(X)$, define the grid representation $E_{\mathrm{grid},R}(f)=\left(\frac{1}{|C_{ij}|}\int_{C_{ij}} f(x)\,dx\right)_{i,j=1}^{R}.$
Then, for every $R\in\mathbb N$ and every decoder
$\mathcal G\colon\mathbb R^{R^2}\to L^2(X)$,

$$
\sup_{\substack{f\in C(X)\\ \|f\|_{\infty,X}\le 1}}
\left\|
f-\mathcal G\!\left(E_{\mathrm{grid},R}(f)\right)
\right\|_{L^2(X)}
\ge
\frac{c}{2}.
$$
\end{theorem}
Increasing the complexity of the decoder cannot recover information already lost by the fixed grid.

\subsection{Universal approximation with a fixed Fourier encoder}
\label{sec:theory:realization}
Fourier approximation allows the frequencies to be chosen according to the target function and desired accuracy. 
In the default \ours{} setup, however, the frequencies in \(\Omega\) are sampled once and fixed without access to the target function, but can also be made learnable as a straightforward extension (Table~\ref{tab:learned_frequency_ablation}).
We therefore ask whether fixing \(\Omega\) limits expressivity. The next result shows that a fixed set of randomly sampled frequencies, when followed by the Transformer, suffices for universal approximation of continuous diagram-level functions.

Fix \(c>0\) and \(q\in\mathbb N\). Let \(\mathcal H_c=[0,c]^2\cap\{b\le d\}\), and let \(\mathcal D_c\) be the set of all multisets of at most \(q\) points in \(\mathcal H_c\). The empty diagram is included and diagonal points are retained. Convergence requires eventual equality of cardinalities and convergence of the matched points.

\begin{theorem}[Universal approximation with a fixed Fourier encoder]
\label{thm:ppe_fixed_encoder_universality}
Let $\Omega$ have at least three rows sampled i.i.d. from a
nondegenerate Gaussian distribution on $\mathbb R^2$. 
For almost every \(\Omega\), for every continuous function \(F\colon\mathcal D_c\to\mathbb R\) and every \(\varepsilon>0\), there exists a two-layer \ours{} Transformer \(T_{\theta,\Omega}\) with finite FFN widths and trainable parameters \(\theta\) such that
$$
\sup_{D\in\mathcal D_c}
\left|T_{\theta,\Omega}(D)-F(D)\right|
<
\varepsilon.
$$
\end{theorem}

\section{Experiments}
\label{sec:experiments}

\subsection{Experimental setup}
\label{sec:setup}

\smallskip
\noindent\textbf{Datasets.}
We evaluate~\ours{} on ORBIT5K, standard TUDataset graph-classification benchmarks used by prior persistence-diagram baselines, and two larger-scale benchmarks, ogbg-molhiv and MCB-C.

\begin{itemize}[label={},leftmargin=1.2em]
\item\textit{ORBIT5K.} The standard point-cloud benchmark for neural TDA methods~\citep{carriere2020perslay}, with 5{,}000 linked twist map orbits (5 classes, $r \in \{2.5, 3.5, 4.0, 4.1, 4.3\}$, 1{,}000 points per orbit). Prior-work numbers use their original protocols. Our reproduced results use the unified 10-fold train/validation/test protocol described below.

\item\textit{TU graph benchmarks.} We evaluate on standard TUDataset benchmarks~\citep{morris2020tudataset}, namely MUTAG, PROTEINS, COX2, DHFR, IMDB-B, IMDB-M, NCI1, and NCI109. These cover molecular graphs and social interaction graphs. We use class labels only for supervised evaluation, and node attributes are not used in the reproduced PD-only comparison.

\item\textit{Larger-scale benchmarks.} We additionally evaluate frozen model configurations on ogbg-molhiv \citep{hu2020ogb} and MCB-C~\citep{graf2025flood}. The ogbg-molhiv benchmark contains 41{,}127 molecular graphs and uses the official scaffold split with ROC-AUC. MCB-C contains 1{,}745 high-resolution point clouds and uses the ten released folds with balanced accuracy. We deterministically sample 20{,}000 input points per MCB-C example before computing persistence.
\end{itemize}

\smallskip 
\noindent\textbf{Persistence diagram computation.}
For point cloud data (ORBIT, linked twist), we compute persistence diagrams via Alpha complex filtration using \texttt{GUDHI}~\citep{gudhi2015}. For graph data, we follow~\perslay{}~\citep{carriere2020perslay} and use Heat Kernel Signature (HKS) sublevel persistence~\citep{sun2009hks}: for each graph $G=(V,E)$ we compute the normalized Laplacian $L_{\text{graph}} = I - D_{\deg}^{-1/2}A_{\text{adj}}D_{\deg}^{-1/2}$, derive HKS values $\text{hks}(v) = \sum_k e^{-t\lambda_k}\psi_k(v)^2$ at diffusion time $t=10$, and use these as the filtration function. 
Here $D_{\deg}$ and $A_{\text{adj}}$ denote degree and adjacency matrices, and $(\lambda_k,\psi_k)$ are Laplacian eigenpairs of $L_{\text{graph}}$.
In both cases, infinite-death points are removed and coordinates are used as raw $(b, d)$ pairs. The comparison in Table~\ref{tab:graph_benchmarks}(a) uses ordinary $H_0$ and $H_1$ persistence. The larger-scale comparison in Table~\ref{tab:graph_benchmarks}(b) uses ordinary $H_0$ and $H_1$ persistence for ogbg-molhiv and Alpha persistence in $H_0$, $H_1$, and $H_2$ for MCB-C.

\smallskip
\noindent\textbf{Baselines.}
All baseline methods are PD-to-prediction architectures, so a fair comparison requires holding the persistence diagram itself constant. Prior papers each use different filtration choices, persistence variants, coordinate conventions, and evaluation splits, which makes their reported numbers difficult to compare on equal footing. We therefore re-implement Transformer-based PD methods inside our unified pipeline as our primary baselines, namely~\xpert{},~\persformer{}, and~\mst{}. 
The reproduced model baselines consume the same ordinary-persistence PDs (Alpha complex on ORBIT, HKS sublevel filtration at $t{=}10$ on TU graphs, both via GUDHI), go through identical data loading, and are evaluated under the same stratified 10-fold cross-validation protocol described in Section~\ref{sec:experiments}. For the \mst{} reproduction, we use ordinary persistence diagrams without the optional DBSCAN clustering step; points with identical coordinates are still counted by multiplicity in the multiset attention. Per-method hyperparameters are tuned on the validation split via random grid search. The evaluated architectural choices and search ranges are reported in Appendix~\ref{app:grid_search}.

\smallskip
\noindent\textbf{\ours{} configuration.}
Architecture and training defaults are listed in Appendix~\ref{app:architecture} (Table~\ref{tab:hyperparams}). Per-dataset learning rate, weight decay, and dropout are tuned via random grid search on the validation split, while batch size, warmup, total epochs, and early-stopping patience are fixed per dataset family (full protocol in~\Aref{app:grid_search}).

\smallskip
\noindent\textbf{Evaluation protocol.}
Prior work on PD classification evaluates on PDs computed under different filtration settings, so observed accuracy gaps between methods reflect a mixture of representation-learning differences and input differences, and a stronger filtration paired with a weaker model can outperform the reverse. Without controlling the input, no comparison can isolate the contribution of the architecture itself. Compounding this, evaluation protocols themselves vary across prior work and often lack an independent held-out test set after model and hyperparameter selection, allowing test-set-informed model selection to inflate reported numbers. Motivated by recent calls for transparent and comparable experimental standards in topological machine learning~\citep{barnes2021comparative,castillarieso2026standards}, we therefore unify both inputs (identical PDs across all methods) and protocol.

Unless otherwise noted, we use stratified 10-fold cross-validation with an explicit train/validation/test split per fold ($\sim$$80/10/10$). Train, validation, and test metrics are recorded at each epoch. Hyperparameter and checkpoint selection use only validation accuracy, and we report the test metric at the validation-selected epoch. Each neural configuration in the unified sweep is run as a full 5-seed $\times$ 10-fold sweep, and the configuration with the highest mean validation accuracy is reported. All random seeds used for data splitting, parameter initialization, and minibatch shuffling are fixed in the released configurations.

\subsection{Benchmark comparison}
\label{sec:results}

\begin{table}[t]
\centering
\caption{\textbf{Benchmark and controlled-interface results.} (a) reports the unified Transformer-based comparison, (b) additional larger-scale evaluations, and (c) a matched $2\times2$ control, whose $\Delta$ column reports the signed difference (Transformer aggregation minus pooled MLP aggregation) in percentage points. Gray shading and \textbf{bold} mark the highest mean, while \underline{underline} marks the second-highest mean within each comparison.}
\label{tab:graph_benchmarks}
\scriptsize
\setlength{\tabcolsep}{2.4pt}
\resizebox{\linewidth}{!}{%
\begin{tabular}{@{}lcccccccccc@{}}
\toprule
\multicolumn{11}{@{}l}{\textbf{(a) Transformer-based comparison}} \\
\cmidrule(r){1-1}\cmidrule(lr){2-2}\cmidrule(l){3-11}
\textbf{Method} & \textbf{ORBIT5K} & \textbf{MUTAG} & \textbf{PROTEINS} & \textbf{COX2} & \textbf{DHFR} & \textbf{IMDB-B} & \textbf{IMDB-M} & \textbf{NCI1} & \textbf{NCI109} & \textbf{TU8} \\
\cmidrule(r){1-1}\cmidrule(lr){2-2}\cmidrule(l){3-11}
\mst{} & 81.56{\tiny$\pm$5.53} & 78.99{\tiny$\pm$7.90} & 68.10{\tiny$\pm$2.89} & 77.53{\tiny$\pm$1.56} & 60.34{\tiny$\pm$3.53} & \underline{67.78{\tiny$\pm$4.34}} & 39.28{\tiny$\pm$3.09} & 65.44{\tiny$\pm$1.28} & 59.97{\tiny$\pm$1.29} & 64.68 \\
\persformer{} & 76.90{\tiny$\pm$13.58}$^{*}$ & 82.48{\tiny$\pm$6.23} & 71.23{\tiny$\pm$3.18} & \underline{78.16{\tiny$\pm$0.58}} & 60.83{\tiny$\pm$0.65} & 67.44{\tiny$\pm$2.75} & 39.63{\tiny$\pm$3.73} & 64.10{\tiny$\pm$3.21} & 64.64{\tiny$\pm$1.66} & 66.06 \\
\xpert{} & \underline{83.86{\tiny$\pm$1.24}} & \underline{83.38{\tiny$\pm$5.75}} & \underline{72.29{\tiny$\pm$3.59}} & 77.51{\tiny$\pm$3.05} & \underline{75.80{\tiny$\pm$3.27}} & 66.48{\tiny$\pm$4.13} & \underline{46.15{\tiny$\pm$3.24}} & \underline{66.36{\tiny$\pm$1.14}} & \underline{66.56{\tiny$\pm$1.02}} & \underline{69.32} \\
\ours{} (Ours) & \cellcolor{gray!15}\textbf{90.68{\tiny$\pm$0.72}} & \cellcolor{gray!15}\textbf{87.44{\tiny$\pm$3.40}} & \cellcolor{gray!15}\textbf{72.96{\tiny$\pm$4.61}} & \cellcolor{gray!15}\textbf{78.70{\tiny$\pm$3.47}} & \cellcolor{gray!15}\textbf{80.03{\tiny$\pm$3.18}} & \cellcolor{gray!15}\textbf{70.24{\tiny$\pm$3.33}} & \cellcolor{gray!15}\textbf{46.60{\tiny$\pm$3.45}} & \cellcolor{gray!15}\textbf{75.66{\tiny$\pm$1.47}} & \cellcolor{gray!15}\textbf{74.09{\tiny$\pm$1.66}} & \cellcolor{gray!15}\textbf{73.22} \\
\bottomrule
\end{tabular}
}
\vspace{0.4mm}
\noindent\parbox{\linewidth}{\raggedright\scriptsize$^{*}$\persformer{} trains very unstably on ORBIT5K. \citet{kim2024xpert} reproduce it with the authors' official code and report the model as ``not trainable'' on ORBIT5K. Figure~\ref{fig:encoder_ablation}(\subref{fig:orbit5k_encoder_convergence}) shows the validation curves.}
\scriptsize
\vspace{1.4mm}
\begin{minipage}[t]{0.49\linewidth}
\centering
\setlength{\tabcolsep}{2.4pt}
\begin{tabular*}{\linewidth}{@{\extracolsep{\fill}}lcc@{}}
\multicolumn{3}{@{}l}{\textbf{(b) Larger-scale benchmarks}} \\
\toprule
\textbf{Method} & \makecell{\textbf{ogbg-molhiv}\\[-1pt]\textbf{ROC-AUC}} & \makecell{\textbf{MCB-C}\\[-1pt]\textbf{Balanced acc.}} \\
\midrule
\mst{} & 63.13{\tiny$\pm$2.72} & 49.30{\tiny$\pm$5.64} \\
\persformer{} & 63.00{\tiny$\pm$0.47} & 14.69{\tiny$\pm$10.66} \\
\xpert{} & \underline{68.27{\tiny$\pm$0.71}} & \underline{56.33{\tiny$\pm$4.28}} \\
\ours{} (Ours) & \cellcolor{gray!15}\textbf{69.32{\tiny$\pm$1.51}} & \cellcolor{gray!15}\textbf{65.79{\tiny$\pm$4.27}} \\
\bottomrule
\end{tabular*}
\end{minipage}%
\hfill
\begin{minipage}[t]{0.49\linewidth}
\centering
\setlength{\tabcolsep}{2.4pt}
\begin{tabular*}{\linewidth}{@{\extracolsep{\fill}}lccc@{}}
\multicolumn{4}{@{}l}{\textbf{(c) Matched representation $\times$ backbone control (TU8)}} \\
\toprule
\textbf{Representation} & \makecell{\textbf{Pooled MLP}\\[-1pt]\textbf{aggregation}} & \makecell{\textbf{Transformer}\\[-1pt]\textbf{aggregation}} & $\boldsymbol{\Delta}$ \\
\midrule
\addlinespace[1.5mm]
Persistence image (PI) & \underline{64.41} & \underline{64.97} & +0.56 \\
\addlinespace[1mm]
Random Fourier feature (RFF) & \cellcolor{gray!15}\textbf{72.16} & \cellcolor{gray!15}\textbf{72.63} & +0.47 \\
\addlinespace[1mm]
RFF $-$ PI & +7.75 & +7.66 &  \\
\addlinespace[0.5mm]
\bottomrule
\end{tabular*}
\end{minipage}
\end{table}

\smallskip 
\noindent\textbf{Classification benchmarks.}
Table~\ref{tab:graph_benchmarks}(a) compares Transformer-based PD methods on ORBIT5K and eight TU graph benchmarks.
\ours{} achieves the highest mean on all nine benchmarks, reaching 90.68\% on ORBIT5K and a TU8 macro accuracy of 73.22\%.
The gains over the strongest reproduced Transformer baseline are 6.82~pp on ORBIT5K and 3.90~pp in TU8 macro accuracy, with particularly large improvements of 9.30 and 7.53~pp on NCI1 and NCI109, respectively.
\mbox{Table~\ref{tab:graph_benchmarks}(b)} shows that \ours{} also achieves the highest mean among the evaluated Transformer-based PD methods on the two larger-scale benchmarks ogbg-molhiv and MCB-C under the separate frozen-transfer protocols. On MCB-C, the margin over the strongest evaluated Transformer-based PD baseline reaches 9.46~pp in balanced accuracy, with no target-specific hyperparameter search.
In the matched control (Table~\ref{tab:graph_benchmarks}(c), Table~\ref{tab:pi_transformer_dataset_control}), replacing PIs with fixed RFF raises TU8 macro accuracy from 64.41 to 72.16 with pooled MLP aggregation and from 64.97 to 72.63 with Transformer aggregation. With the RFF tokenizer fixed, pooled MLP and Transformer aggregation reach TU8 macro accuracies of 72.16 and 72.63, respectively.
The representation change therefore improves accuracy by 7.66 to 7.75~pp, while switching from pooled MLP to Transformer aggregation yields much smaller gains of 0.56~pp with PIs and 0.47~pp with RFF. The gain thus comes mostly from the point representation rather than from the downstream model.

\ours{} achieves the highest TU8 macro accuracy of 73.22 among the evaluated method--downstream configurations, spanning Transformer-based PD methods, learned representations, fixed summaries, and persistence kernels (Table~\ref{tab:tu8_method_summary}).
This result demonstrates that fixed RFF point tokens paired with a standard Transformer provide a simple and effective architecture for PD classification.
The comparison uses shared PD inputs and outer folds, with model selection based only on validation data.
Per-dataset results and downstream configurations are provided in Table~\ref{tab:fixed_vector_full} (Appendix~\ref{app:additional_results}).

\begin{table}[!htbp]
\centering
\caption{\textbf{TU8 comparison across PD method families.} Each row is a method--downstream configuration. TU8 is the unweighted mean over the eight graph datasets. Gray shading and \textbf{bold} mark the highest mean, and \underline{underline} marks the second-highest mean.}
\label{tab:tu8_method_summary}
\scriptsize
\renewcommand{\arraystretch}{1.06}
\begin{minipage}[t]{0.49\linewidth}
\centering
\setlength{\tabcolsep}{2.2pt}
\begin{tabular*}{\linewidth}{@{\extracolsep{\fill}}llr@{}}
\toprule
\textbf{Method} & \textbf{Learning model} & \textbf{TU8} \\
\midrule
\multicolumn{3}{@{}l}{\textbf{Transformer-based PD methods}} \\
\ours{} & Transformer & \cellcolor{gray!15}\textbf{73.22} \\
\xpert{} & Transformer & 69.32 \\
\persformer{} & Transformer & 66.06 \\
\mst{} & Transformer & 64.68 \\
\addlinespace[0.5mm]
\multicolumn{3}{@{}l}{\textbf{Learned PD representations}} \\
ATOL & Random forest & \underline{72.56} \\
\perslay{}$^\star$ & Learned invariant layer & 71.10 \\
RationalHat & MLP & 68.80 \\
RationalHat & Graph MLP & 64.42 \\
\bottomrule
\end{tabular*}
\end{minipage}%
\hfill
\begin{minipage}[t]{0.49\linewidth}
\centering
\setlength{\tabcolsep}{2.2pt}
\begin{tabular*}{\linewidth}{@{\extracolsep{\fill}}llr@{}}
\toprule
\textbf{Method} & \textbf{Learning model} & \textbf{TU8} \\
\midrule
\multicolumn{3}{@{}l}{\textbf{Fixed PD summaries}} \\
Persistence image & Random forest & 71.67 \\
Persistence landscape & Random forest & 71.17 \\
Betti curve & Random forest & 71.01 \\
Persistence silhouette & Random forest & 68.51 \\
\addlinespace[3.5mm]
\multicolumn{3}{@{}l}{\textbf{Persistence kernels}} \\
Sliced Wasserstein & SVM & 71.34 \\
Persistence scale-space & SVM & 69.56 \\
Persistence-weighted Gaussian & SVM & 68.23 \\
\bottomrule
\end{tabular*}
\end{minipage}
\vspace{0.6mm}
\parbox{0.97\linewidth}{\raggedright$^\star$ We evaluate Image, Line/Permutation-Equivariant, Landscape, and Exponential variants under the same protocol and report the variant selected using validation data. Image is selected on all datasets. Full variant-wise results are reported in Table~\ref{tab:perslay_variants}.}
\end{table}

\begin{table}[tbp]
\centering
\caption{\textbf{Point-representation and downstream-module controls.} Accuracy (\%) under the matched TU8 protocol. Shared configurations are repeated across controls so that each comparison is self-contained. TU8 is the unweighted mean over the eight graph datasets. Gray shading and \textbf{bold} mark the highest mean, while \underline{underline} marks the second-highest mean within each control.}
\label{tab:full_interface_controls}
\scriptsize \setlength{\tabcolsep}{1.8pt}
\resizebox{\linewidth}{!}{%
\begin{tabular}{@{}lllccccccccc@{}}
\toprule
\makecell[l]{\textbf{Varied}\\[-1pt]\textbf{factor}} & \makecell[l]{\textbf{Point}\\[-1pt]\textbf{representation}} & \makecell[l]{\textbf{Downstream}\\[-1pt]\textbf{module}} & \textbf{MUTAG} & \textbf{PROTEINS} & \textbf{COX2} & \textbf{DHFR} & \textbf{IMDB-B} & \textbf{IMDB-M} & \textbf{NCI1} & \textbf{NCI109} & \textbf{TU8} \\
\midrule
\multirow{4}{*}{Both} & Point MLP & Pooled MLP (mean+sum) & 81.38 & \cellcolor{gray!15}\textbf{72.58} & 78.08 & 60.32 & 64.28 & 43.56 & 62.46 & 62.84 & 65.69 \\
& Point MLP & Transformer & 82.69 & 71.48 & \underline{78.16} & 60.64 & 68.64 & 45.68 & 64.56 & 65.29 & 67.14 \\
& Fixed RFF & Pooled MLP (mean+sum) & \cellcolor{gray!15}\textbf{85.74} & 71.61 & 77.46 & \cellcolor{gray!15}\textbf{79.85} & \underline{69.00} & \underline{46.43} & \underline{73.64} & \underline{73.58} & \underline{72.16} \\
& Fixed RFF & Transformer & \underline{85.38} & \underline{71.72} & \cellcolor{gray!15}\textbf{78.30} & \underline{79.56} & \cellcolor{gray!15}\textbf{69.88} & \cellcolor{gray!15}\textbf{46.47} & \cellcolor{gray!15}\textbf{75.67} & \cellcolor{gray!15}\textbf{74.07} & \cellcolor{gray!15}\textbf{72.63} \\
\midrule
\multirow{2}{*}{\makecell[l]{Point\\representation}} & Raw coordinates & DeepSets & \underline{80.43} & \underline{71.82} & \cellcolor{gray!15}\textbf{78.16} & \underline{60.85} & \underline{63.84} & \underline{39.45} & \underline{62.41} & \underline{62.42} & \underline{64.92} \\
& Fixed RFF & DeepSets & \cellcolor{gray!15}\textbf{86.08} & \cellcolor{gray!15}\textbf{72.09} & \underline{77.55} & \cellcolor{gray!15}\textbf{79.37} & \cellcolor{gray!15}\textbf{68.56} & \cellcolor{gray!15}\textbf{46.67} & \cellcolor{gray!15}\textbf{74.08} & \cellcolor{gray!15}\textbf{73.22} & \cellcolor{gray!15}\textbf{72.20} \\
\midrule
\multirow{5}{*}{\makecell[l]{Downstream\\module}} & Fixed RFF & Pooled MLP (mean) & \cellcolor{gray!15}\textbf{86.80} & 70.15 & 76.72 & 78.82 & 68.26 & 46.19 & 69.44 & 68.98 & 70.67 \\
& Fixed RFF & Pooled MLP (sum) & \underline{86.67} & 70.20 & 77.37 & \cellcolor{gray!15}\textbf{80.49} & 68.04 & \underline{46.48} & 73.30 & 73.02 & 71.95 \\
& Fixed RFF & DeepSets & 86.08 & \cellcolor{gray!15}\textbf{72.09} & 77.55 & 79.37 & 68.56 & \cellcolor{gray!15}\textbf{46.67} & \underline{74.08} & \underline{73.22} & \underline{72.20} \\
& Fixed RFF & Learned-query pooling & 85.06 & 70.15 & \underline{77.62} & 79.27 & \underline{69.40} & 45.97 & 72.15 & 71.61 & 71.40 \\
& Fixed RFF & Transformer & 85.38 & \underline{71.72} & \cellcolor{gray!15}\textbf{78.30} & \underline{79.56} & \cellcolor{gray!15}\textbf{69.88} & 46.47 & \cellcolor{gray!15}\textbf{75.67} & \cellcolor{gray!15}\textbf{74.07} & \cellcolor{gray!15}\textbf{72.63} \\
\bottomrule
\end{tabular}
}
\end{table}

\vspace{-2mm}
\subsection{Control Analyses and Validation Behavior}
\label{sec:encoder_ablation}

The matched control in Table~\ref{tab:graph_benchmarks}(c) compares an image-based representation with pointwise features. 
Here we extend this comparison to other pointwise encoders and ask how accuracy varies across downstream modules with the encoder held fixed as shown in Table~\ref{tab:full_interface_controls}.
Point MLP is a two-layer encoder over raw $(b,d)$ coordinates, while the raw-coordinate representation uses only the direct coordinate projection. Pooled MLP applies a two-layer MLP after permutation-invariant pooling, using sum and mean together in the factorial control and one of the two alone in the downstream control. The remaining downstream modules are DeepSets~\citep{zaheer2017deepsets}, learned-query multi-head pooling, and the~\ours{} Transformer with CLS readout. Within each matched control, the homology-dimension embedding, classifier head, optimizer, and data pipeline are fixed.

\begin{figure}[tbp]
\centering
\begin{subfigure}[b]{0.40\linewidth}
\centering
\footnotesize
\setlength{\tabcolsep}{4.0pt}
\begin{tabular}{lccc}
\toprule
\textbf{Dataset} & \textbf{Fourier} & \textbf{MLP} & $\boldsymbol{\Delta}$ \\
\midrule
MUTAG & \textbf{85.38} & 82.69 & +2.69 \\
PROTEINS & \textbf{71.72} & 71.48 & +0.24 \\
COX2 & \textbf{78.30} & 78.16 & +0.14 \\
DHFR & \textbf{79.56} & 60.64 & +18.92 \\
IMDB-B & \textbf{69.88} & 68.64 & +1.24 \\
IMDB-M & \textbf{46.47} & 45.68 & +0.79 \\
NCI1 & \textbf{75.67} & 64.56 & +11.11 \\
NCI109 & \textbf{74.07} & 65.29 & +8.78 \\
\midrule
TU8 & \textbf{72.63} & 67.14 & +5.49 \\
\bottomrule
\end{tabular}
\caption{TU graph benchmarks.}
\label{tab:encoder_ablation}
\end{subfigure}
\hfill
\begin{subfigure}[b]{0.58\linewidth}
\centering
\includegraphics[width=\linewidth]{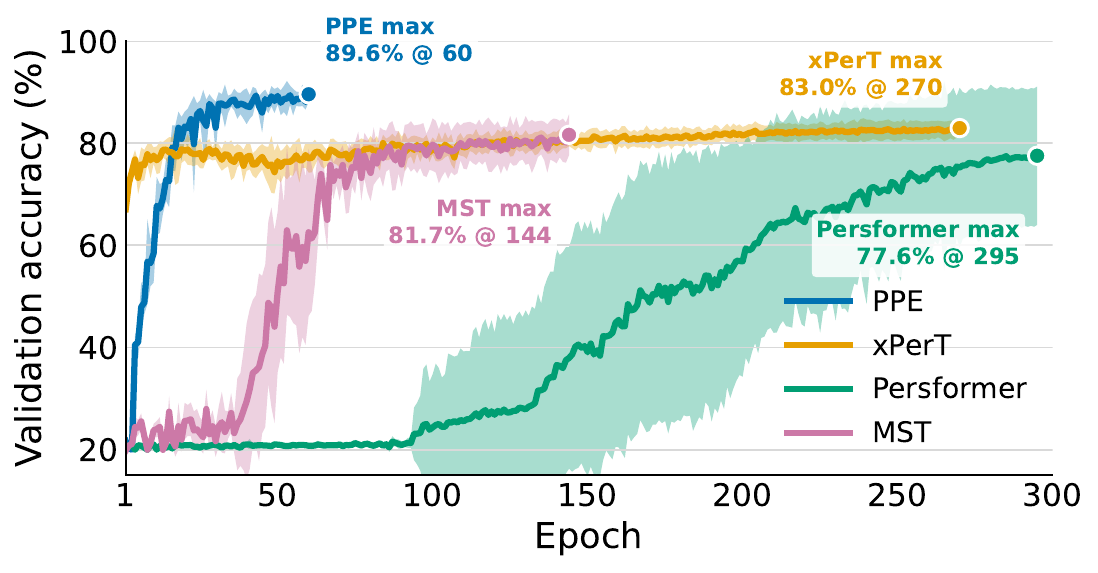}
\caption{Validation convergence on ORBIT5K.}
\label{fig:orbit5k_encoder_convergence}
\end{subfigure}
\caption{\textbf{Coordinate-encoder ablation and validation convergence.} (\subref{tab:encoder_ablation}) Accuracy (\%) for fixed Fourier versus learned point-MLP encoders in the same~\ours{} backbone on TU graph benchmarks, where \textbf{bold} marks the better encoder per row. 
(\subref{fig:orbit5k_encoder_convergence}) Validation accuracy versus epoch on ORBIT5K under the unified protocol. Each line is the mean over the ten cross-validation folds and each band spans those folds. Each label gives a method's peak mean accuracy and the epoch at which it occurs.}
\label{fig:encoder_ablation}
\end{figure}
Across all three matched point-representation controls, RFF raises TU8 by 6.47~pp with mean+sum pooling, by 5.49~pp with the Transformer, and by 7.28~pp over raw coordinates with DeepSets. Changing the downstream backbone shifts TU8 by 1.45~pp with Point MLP and by 0.47~pp with RFF, while the five fixed-RFF downstream modules span 1.96~pp. RFF is higher on at least six of eight datasets in each point-representation control and on all eight with the Transformer, with the largest single-dataset gain reaching 19.5~pp on DHFR under mean+sum pooling. In Table~\ref{tab:full_interface_controls}, every point-encoder replacement produces a larger TU8 macro gain than any backbone or pooling change, which makes the encoder the more consequential of the two choices tested.
Figure~\ref{fig:encoder_ablation}(\subref{tab:encoder_ablation}) breaks the two Transformer rows of Table~\ref{tab:full_interface_controls} down by dataset.
Figure~\ref{fig:encoder_ablation}(\subref{fig:orbit5k_encoder_convergence}) further shows the validation dynamics on ORBIT5K. PPE reaches the highest and earliest peak, at 89.6\% by epoch 60, while MST and \xpert{} peak later at 81.7--83.0\%. \persformer{} shows large fold-to-fold variation and reaches 77.6\% at epoch 295. Consistently, \citet{kim2024xpert} reproduced \persformer{} with the authors' official code and reported it as ``not trainable'' on ORBIT5K.
An ORBIT5K coordinate-noise sweep, where the Fourier encoder wins every level, is reported in Figure~\ref{fig:noise_sweep} (Appendix~\ref{app:noise_sweep}).

\subsection{Sub-cell Fidelity Discrimination: Plain vs.\ Sprinkle Donuts}
\label{sec:donut_subcell}

Persistent homology records feature lifetimes, but lifetime alone does not determine which coordinate differences matter for classification, and hard binning discards a point's position within its assigned cell. We therefore use plain and sprinkle donuts, which share coarse topology, to test whether a PD interface still exposes class-relevant within-cell variation.

\noindent\textbf{Dataset.}
We construct a binary classification task over two donut variants, namely~\emph{plain} donuts (3D point-cloud renderings of donut shapes, approximately torus-topological) and~\emph{sprinkle} donuts (the same donut shape with sprinkles on the surface). The two classes share identical coarse topology (connected component, dominant loop) and have nearly identical $H_0/H_1$ PDs. The sprinkle signal manifests only as a sub-cell shift in PD-coordinate statistics, where $H_0$ finite death values and $H_1$ birth-coordinate quantiles are slightly larger for sprinkle donuts. A logistic regression on hand-crafted PD statistics reaches 78.0$\pm$9.3\% accuracy under 10-fold CV, confirming the discriminative signal is present in the diagrams. Each point cloud is produced by a text${\to}$image${\to}$3D pipeline and uniformly subsampled to 2048 points. The full construction protocol and per-class counts are in Appendix~\ref{app:donut_subcell}.

\smallskip 
\noindent\textbf{Results.}
Under random hyperparameter search,~\ours{} reaches the highest mean test accuracy among the four evaluated Transformer-based PD methods at 71.97\%, ahead of~\xpert{} (54.11\%),~\persformer{} (52.32\%), and~\mst{} (50.91\%), a margin of 17.86 to 21.06~pp (Table~\ref{tab:donut_classification} in Appendix~\ref{app:donut_subcell}).
Theorem~\ref{thm:grid_capacity} shows that, for uniform cell averaging, a fixed cell partition loses information that no decoder can recover.~\xpert{} applies such a partition to PD points, mapping every point in a cell to the same representative and so discarding its position within the cell.
~\ours{} instead maps each $(b,d)$ point to fixed Fourier features before self-attention, and Figure~\ref{fig:xpert_donut_tokenization} in Appendix~\ref{app:donut_subcell} illustrates the grid tokenization on this benchmark.
The donut result shows that class-discriminative information can reside in sub-cell PD coordinates that hard grid binning discards.

\begin{table}[!htbp]
\begingroup
\caption{\textbf{Fixed-versus-learned frequency ablation.} Mean accuracy (\%) under matched conditions. Fixed keeps the sampled Fourier frequencies frozen, while Learned makes the same initialization trainable. The $1\times$, $10\times$, and $100\times$ labels denote frequency learning-rate multipliers; for $10\times$ and $100\times$, frequency weight decay is scaled inversely. Range is the max–min difference within each row.}
\label{tab:learned_frequency_ablation}
\footnotesize
\setlength{\tabcolsep}{5.5pt}
\noindent\makebox[\linewidth][c]{%
\begin{tabular*}{\linewidth}{@{\extracolsep{\fill}}lccccc@{}}
\toprule
\textbf{Benchmark} & \textbf{Fixed} & \textbf{Learned $1\times$ LR} & \textbf{Learned $10\times$ LR} & \textbf{Learned $100\times$ LR} & \textbf{Range} \\
\midrule
TU8 macro & 72.68 & 72.68 & 72.67 & 72.79 & 0.12 \\
ORBIT5K & 90.52 & 90.34 & 90.26 & 90.16 & 0.36 \\
\bottomrule
\end{tabular*}
}
\endgroup
\end{table}

\section{Discussion and Limitations}
\label{sec:discussion_and_limitations}
\noindent\textbf{Fixed versus learned Fourier frequencies.}
We test whether the Fourier frequencies sampled from a Gaussian distribution should be learned end to end.
Table~\ref{tab:learned_frequency_ablation} compares fixed and learned Fourier frequencies under otherwise identical experimental conditions.
No learned setting consistently improves accuracy across TU8 and ORBIT5K. This result supports the practical use of the sampled-and-frozen frequency design analyzed in Theorem~\ref{thm:ppe_fixed_encoder_universality}.

\noindent\textbf{Limitations.}
Since \ours{} uses a standard Transformer, attention scales as \(O(N^2)\) with \(N\) input tokens, making dense point-cloud PDs such as ORBIT5K more costly.
Capping tokens to the longest-lived points in each homology dimension mitigates this cost. 
On MCB-C, halving this cap from 1,024 to 512 preserves balanced accuracy (65.79 vs. 65.82,~Table~\ref{tab:mcb_cap_sensitivity}). 

\section{Conclusion}
\label{sec:conclusion}
We asked how far a persistence diagram can be carried by the representation of its individual points, leaving the Transformer alone. \ours{} maps each birth--death pair through a fixed random Fourier feature map and a learnable homology-dimension embedding, and hands the resulting tokens to an off-the-shelf encoder. Under the unified Transformer comparison it has the highest mean among the evaluated Transformer-based PD models across the nine accuracy benchmarks, and in the frozen-transfer evaluations it also leads them on both larger-scale benchmarks, while the matched controls move single-dataset accuracy by as much as 19.5~pp. 
At every resolution, uniform grid averaging leaves a reconstruction error that no decoder can bring below a fixed positive bound, whereas a Transformer with fixed randomly sampled Fourier frequencies is almost surely a universal approximator without knowing the target frequencies.
Because \ours{} produces model-width tokens without PD-aware attention or pooling, it drops into token-based Transformer pipelines as they are.

\setcounter{table}{0}
\renewcommand{\thetable}{A\arabic{table}}
\setcounter{figure}{0}
\renewcommand{\thefigure}{A\arabic{figure}}
\setcounter{equation}{0}
\renewcommand{\theequation}{A\arabic{equation}}
\appendix

\begingroup
\section{Additional Experimental Comparisons}
\label{app:additional_results}
\endgroup

The main-text summary in Table~\ref{tab:tu8_method_summary} is complemented here by common-head comparisons. Table~\ref{tab:fixed_vector_full} holds the downstream model fixed within each block and reports every dataset, covering the two common heads, the persistence kernels, and the RationalHat projection. All TU8 rows use the shared prepared ordinary-persistence inputs and outer folds.

\begin{table}[!htbp]
\centering
\caption{\textbf{Per-dataset PI versus RFF comparison with a matched Transformer backbone.} Accuracy (\%) over ten folds and five seeds per fold. Both representations use the same Transformer architecture, folds, and validation-only selection rule. Bold marks the higher mean for each dataset.}
\label{tab:pi_transformer_dataset_control}
\footnotesize \setlength{\tabcolsep}{8pt}
\begin{tabular}{@{}lccc@{}}
\toprule
\textbf{Dataset} & \textbf{PI + Transformer} & \textbf{RFF + Transformer} & \textbf{PI $-$ RFF} \\
\midrule
MUTAG & 81.70 & \textbf{85.38} & $-$3.68 \\
PROTEINS & \textbf{73.22} & 71.72 & +1.50 \\
COX2 & 78.16 & \textbf{78.30} & $-$0.14 \\
DHFR & 60.50 & \textbf{79.56} & $-$19.06 \\
IMDB-B & 50.30 & \textbf{69.88} & $-$19.58 \\
IMDB-M & 43.11 & \textbf{46.47} & $-$3.36 \\
NCI1 & 66.92 & \textbf{75.67} & $-$8.75 \\
NCI109 & 65.83 & \textbf{74.07} & $-$8.24 \\
\bottomrule
\end{tabular}
\end{table}

\begin{table}[!htbp]
\centering
\caption{\textbf{PersLay variant comparison.} All variants use the same PyTorch PersLay implementation, prepared persistence diagrams, outer folds, training protocol, and validation-only selection. TU8 is the unweighted mean over the eight graph datasets. Gray shading and \textbf{bold} mark the highest mean, while \underline{underline} marks the second-highest mean.}
\label{tab:perslay_variants}
\scriptsize \setlength{\tabcolsep}{7pt}
\begin{tabular}{@{}lcc@{}}
\toprule
\textbf{PersLay variant} & \textbf{ORBIT5K} & \textbf{TU8} \\
\midrule
Image & \cellcolor{gray!15}\textbf{81.36} & \cellcolor{gray!15}\textbf{71.10} \\
Line/PEQ & \underline{36.49} & \underline{66.31} \\
Landscape & 32.97 & 64.05 \\
Exponential & 27.65 & 65.94 \\
\bottomrule
\end{tabular}
\end{table}

\begin{table}[!htbp]
\centering
\caption{\textbf{Controlled comparison across TU8 datasets.}
Rows are grouped by downstream model or model family, with matched representations sharing the same downstream model and evaluation protocol.
Per-dataset entries report mean test accuracy and population standard deviation across ten outer folds; TU8 is the unweighted mean over the eight datasets.
Kernel construction uses training diagrams only, and SVM selection uses validation data only.
Gray shading and \textbf{bold} mark the highest mean in each column across all configurations, and \underline{underlining} marks the second highest.}
\label{tab:fixed_vector_full}
\scriptsize \setlength{\tabcolsep}{2.4pt}
\resizebox{\linewidth}{!}{%
\begin{tabular}{@{}lccccccccc@{}}
\toprule
\textbf{Representation} & \textbf{MUTAG} & \textbf{PROTEINS} & \textbf{COX2} & \textbf{DHFR} & \textbf{IMDB-B} & \textbf{IMDB-M} & \textbf{NCI1} & \textbf{NCI109} & \textbf{TU8} \\
\midrule
\multicolumn{10}{@{}l}{\textbf{Common MLP head from the persistence-image configuration}} \\
\midrule
Persistence image & 82.76{\tiny$\pm$6.24} & \cellcolor{gray!15}\textbf{73.47{\tiny$\pm$3.66}} & 77.39{\tiny$\pm$1.44} & 61.16{\tiny$\pm$0.90} & 65.72{\tiny$\pm$3.44} & 43.27{\tiny$\pm$5.17} & 66.01{\tiny$\pm$2.55} & 65.57{\tiny$\pm$1.77} & 66.92 \\
ATOL & 86.54{\tiny$\pm$4.28} & 71.66{\tiny$\pm$3.54} & 77.65{\tiny$\pm$0.79} & 72.20{\tiny$\pm$5.78} & \underline{67.32{\tiny$\pm$6.46}} & 44.03{\tiny$\pm$3.98} & 64.32{\tiny$\pm$1.30} & 66.00{\tiny$\pm$1.57} & 68.71 \\
Persistence landscape & 80.95{\tiny$\pm$10.87} & 66.45{\tiny$\pm$2.08} & 78.08{\tiny$\pm$0.55} & 71.14{\tiny$\pm$5.06} & 64.96{\tiny$\pm$3.06} & 40.99{\tiny$\pm$2.85} & 61.14{\tiny$\pm$1.41} & 60.73{\tiny$\pm$2.22} & 65.56 \\
Persistence silhouette & 66.30{\tiny$\pm$1.34} & 67.53{\tiny$\pm$2.46} & 78.08{\tiny$\pm$0.76} & 60.90{\tiny$\pm$0.40} & 65.44{\tiny$\pm$3.51} & 41.29{\tiny$\pm$2.85} & 56.76{\tiny$\pm$2.19} & 56.98{\tiny$\pm$2.56} & 61.66 \\
Betti curve & 85.19{\tiny$\pm$5.70} & 68.71{\tiny$\pm$2.51} & 77.35{\tiny$\pm$0.89} & 78.18{\tiny$\pm$2.93} & 65.12{\tiny$\pm$2.97} & 40.92{\tiny$\pm$3.37} & 69.49{\tiny$\pm$2.01} & 70.10{\tiny$\pm$1.35} & 69.38 \\
\midrule
\multicolumn{10}{@{}l}{\textbf{Random forest with 100 trees}} \\
\midrule
Persistence image & 83.94{\tiny$\pm$7.46} & \underline{72.09{\tiny$\pm$3.14}} & 78.00{\tiny$\pm$5.42} & 76.93{\tiny$\pm$3.92} & 66.88{\tiny$\pm$4.70} & \underline{46.21{\tiny$\pm$5.01}} & 74.51{\tiny$\pm$1.93} & 74.79{\tiny$\pm$1.86} & \underline{71.67} \\
ATOL & 88.22{\tiny$\pm$3.70} & 70.38{\tiny$\pm$3.72} & \underline{78.84{\tiny$\pm$4.67}} & \underline{79.02{\tiny$\pm$2.98}} & \cellcolor{gray!15}\textbf{71.92{\tiny$\pm$4.80}} & \cellcolor{gray!15}\textbf{47.05{\tiny$\pm$3.63}} & 72.74{\tiny$\pm$1.29} & 72.31{\tiny$\pm$1.64} & \cellcolor{gray!15}\textbf{72.56} \\
Persistence landscape & 85.79{\tiny$\pm$6.55} & 68.41{\tiny$\pm$3.10} & 78.64{\tiny$\pm$5.30} & \cellcolor{gray!15}\textbf{79.57{\tiny$\pm$4.37}} & 65.10{\tiny$\pm$5.69} & 40.29{\tiny$\pm$2.51} & \underline{76.42{\tiny$\pm$1.95}} & \underline{75.11{\tiny$\pm$1.96}} & 71.17 \\
Persistence silhouette & 84.07{\tiny$\pm$8.71} & 66.47{\tiny$\pm$3.24} & 76.01{\tiny$\pm$4.66} & 76.87{\tiny$\pm$4.57} & 65.58{\tiny$\pm$4.41} & 40.20{\tiny$\pm$2.25} & 70.01{\tiny$\pm$1.82} & 68.83{\tiny$\pm$1.53} & 68.51 \\
Betti curve & \cellcolor{gray!15}\textbf{89.03{\tiny$\pm$5.33}} & 68.94{\tiny$\pm$3.07} & 78.17{\tiny$\pm$3.24} & 77.91{\tiny$\pm$4.23} & 65.00{\tiny$\pm$4.49} & 40.52{\tiny$\pm$2.03} & 74.70{\tiny$\pm$1.76} & 73.78{\tiny$\pm$1.98} & 71.01 \\
\midrule
\multicolumn{10}{@{}l}{\textbf{Support vector machine with persistence kernels}} \\
\midrule
Persistence scale-space & 86.95{\tiny$\pm$8.66} & 68.11{\tiny$\pm$2.19} & \cellcolor{gray!15}\textbf{79.25{\tiny$\pm$3.15}} & 76.19{\tiny$\pm$5.84} & 65.50{\tiny$\pm$3.83} & 40.73{\tiny$\pm$3.16} & 69.93{\tiny$\pm$1.68} & 69.78{\tiny$\pm$1.29} & 69.56 \\
Persistence-weighted Gaussian & 80.87{\tiny$\pm$8.83} & 67.03{\tiny$\pm$2.78} & 77.08{\tiny$\pm$3.41} & 75.67{\tiny$\pm$4.30} & 64.80{\tiny$\pm$2.56} & 42.00{\tiny$\pm$2.55} & 70.44{\tiny$\pm$1.65} & 67.99{\tiny$\pm$2.50} & 68.23 \\
Sliced Wasserstein & 86.78{\tiny$\pm$3.98} & 69.45{\tiny$\pm$3.40} & 76.23{\tiny$\pm$2.23} & 78.72{\tiny$\pm$4.99} & 63.90{\tiny$\pm$3.14} & 41.80{\tiny$\pm$3.39} & \cellcolor{gray!15}\textbf{77.44{\tiny$\pm$2.41}} & \cellcolor{gray!15}\textbf{76.42{\tiny$\pm$2.17}} & 71.34 \\
\midrule
\multicolumn{10}{@{}l}{\textbf{RationalHat projection with an MLP head}} \\
\midrule
RationalHat (graph MLP) & 80.29{\tiny$\pm$6.17} & 69.27{\tiny$\pm$2.89} & 78.16{\tiny$\pm$0.58} & 61.17{\tiny$\pm$2.73} & 64.12{\tiny$\pm$3.79} & 40.55{\tiny$\pm$2.95} & 61.65{\tiny$\pm$2.55} & 60.14{\tiny$\pm$2.00} & 64.42 \\
RationalHat (MLP) & \underline{88.39{\tiny$\pm$6.64}} & 70.74{\tiny$\pm$3.32} & 78.16{\tiny$\pm$0.58} & 71.42{\tiny$\pm$5.51} & 64.36{\tiny$\pm$3.94} & 40.51{\tiny$\pm$4.12} & 68.25{\tiny$\pm$2.35} & 68.56{\tiny$\pm$1.40} & 68.80 \\
\bottomrule
\end{tabular}
}
\end{table}

Table~\ref{tab:large_scale_full} expands the larger-scale comparison in Table~\ref{tab:graph_benchmarks} to the evaluated PD methods. The ogbg-molhiv experiment uses 41{,}127 graphs, the official scaffold split, and five seeds. The MCB-C experiment uses 1{,}745 samples, ten predefined folds, and five seeds per fold. No target-specific hyperparameter sweep is performed for these frozen-transfer evaluations.

\begin{table}[!htbp]
\centering
\caption{\textbf{Full larger-scale comparison.} The ogbg-molhiv comparison reports ROC-AUC and the MCB-C comparison reports balanced accuracy. Values are pooled seed-run mean and population standard deviation. Models marked $^\dagger$ use the original graph or point cloud rather than a persistence diagram and are included only as task-level references, and are excluded from the ranking. The graph references are GIN~\citep{xu2019gin} and GCN~\citep{kipf2017gcn}, and the point-cloud reference follows PointNet~\citep{qi2017pointnet}. Gray shading and \textbf{bold} mark the highest mean in each column, while \underline{underline} marks the second-highest.}
\label{tab:large_scale_full}
\scriptsize
\begin{minipage}[t]{0.49\linewidth}
\centering
\setlength{\tabcolsep}{3pt}
\begin{tabular*}{\linewidth}{@{\extracolsep{\fill}}lr@{}}
\toprule
\multicolumn{2}{@{}l}{\textbf{ogbg-molhiv ROC-AUC}} \\
\midrule
\textbf{Method} & \textbf{Score} \\
\midrule
\perslay{} & \cellcolor{gray!15}\textbf{70.92{\tiny$\pm$0.30}} \\
Persistence image + MLP & \underline{70.38{\tiny$\pm$0.42}} \\
\ours{} & 69.32{\tiny$\pm$1.51} \\
\xpert{} & 68.27{\tiny$\pm$0.71} \\
ATOL + random forest & 67.71{\tiny$\pm$0.74} \\
\mst{} & 63.13{\tiny$\pm$2.72} \\
\persformer{} & 63.00{\tiny$\pm$0.47} \\
\midrule
Molecular GIN$^\dagger$ & 73.10{\tiny$\pm$0.69} \\
Molecular GCN$^\dagger$ & 70.06{\tiny$\pm$1.08} \\
\bottomrule
\end{tabular*}
\end{minipage}%
\hfill
\begin{minipage}[t]{0.49\linewidth}
\centering
\setlength{\tabcolsep}{3pt}
\begin{tabular*}{\linewidth}{@{\extracolsep{\fill}}lr@{}}
\toprule
\multicolumn{2}{@{}l}{\textbf{MCB-C balanced accuracy}} \\
\midrule
\textbf{Method} & \textbf{Score} \\
\midrule
\ours{} & \cellcolor{gray!15}\textbf{65.79{\tiny$\pm$4.27}} \\
ATOL + random forest & \underline{64.40{\tiny$\pm$2.65}} \\
Persistence image + MLP & 58.61{\tiny$\pm$4.36} \\
\xpert{} & 56.33{\tiny$\pm$4.28} \\
\mst{} & 49.30{\tiny$\pm$5.64} \\
\perslay{} & 49.15{\tiny$\pm$3.81} \\
\persformer{} & 14.69{\tiny$\pm$10.66} \\
\midrule
\makecell[l]{PointNet-style$^\dagger$\\[-1pt]20k input points} & \makecell[r]{63.76{\tiny$\pm$4.88}\\[-1pt]\strut} \\
\bottomrule
\end{tabular*}
\end{minipage}
\end{table}

\section{Related Work}
\label{app:related}

\noindent\textbf{Applications of Persistent Homology in the Sciences.}
Persistent homology has produced application-driven evidence across scientific domains. In molecular and drug-discovery contexts, topological invariants of biomolecular complexes inform property prediction, ligand--target scoring, and virtual screening~\citep{cangwei2017topologynet,weexia2021,cangwuwei2018drugdesign,nguyen2019d3r,townsend2020}. In materials science, persistence captures the hierarchical pore architectures that govern bulk mechanical and transport properties of amorphous and granular media~\citep{hiraoka2016amorphous,saadatfar2017granular}. In dynamical systems, sliding-window embeddings and time-series persistence recover attractor topology and detect parameter changes such as machining-chatter onset~\citep{pereaharer2015sliding,khasawnehmunch2016chatter}. In structural biology, persistence quantifies binding-state separation and conformational compressibility~\citep{kovacevnikolic2016protein,gameiro2015protein}, and topological fingerprints of protein structure model flexibility and predict folding stability~\citep{xia2014persistent}. In neuroscience, persistent homology and time-varying ``vineyard'' analyses reveal higher-order brain-connectivity structure across cognitive states, individuals, and clinical populations~\citep{petri2014homological,saggar2018dynamical,sizemore2019importance,yoo2016vineyard,ghahremani2018alteration}. In graph machine learning, persistence features alone can match GNN performance on link-prediction benchmarks, complementing message passing rather than duplicating it~\citep{you2024phlp}. These applications motivate a reusable PD-to-token interface for settings where a Transformer is the intended downstream backbone.

\smallskip 
\noindent\textbf{Persistence Diagrams and Machine Learning.}
Several vectorization approaches address the multiset/variable-size mismatch by moving persistence diagrams into spaces compatible with standard learning algorithms. Persistence landscapes~\citep{bubenik2015statistical} map a diagram to a sequence of piecewise-linear functions in a Banach space, enabling averages, limit theorems, hypothesis tests, and lower bounds for bottleneck and Wasserstein distances. Persistence images~\citep{adams2017persistence} instead place persistence-weighted Gaussian mass on each diagram point and integrate over a fixed grid, yielding stable finite-dimensional vectors that work directly with vector-space classifiers. Persistence silhouettes~\citep{chazal2015silhouette} are a weighted functional summary of the same family with an analogous convergence theory, and Betti curves~\citep{umeda2017betti} record the number of features alive at each filtration value. Kernel methods such as the stable multi-scale kernel of~\citet{reininghaus2015stable}, the persistence-weighted Gaussian kernel of~\citet{kusano2016pwg}, and the sliced Wasserstein kernel of~\citet{carriere2017sliced} avoid an explicit finite vectorization and compare diagrams through a positive-definite kernel with stability guarantees under the 1-Wasserstein distance. These methods made persistence diagrams usable in classical machine-learning pipelines, but their stability is obtained through functional smoothing, rasterization, or pairwise kernel comparison.~\ours{} targets the complementary setting where a neural encoder consumes the diagram points directly while preserving fine coordinate information.

\smallskip 
\noindent\textbf{Deep Learning on Persistence Diagrams.}
\perslay{}~\citep{carriere2020perslay} introduced permutation-invariant layers for PDs (87.7\% on ORBIT5K); PLLay~\citep{kim2020pllay} learned persistence landscapes end-to-end. \persformer{}~\citep{reinauer2021persformer} uses a learned pointwise coordinate embedding, stacked self-attention, learned-query multi-head attention pooling, and an MLP decoder, while its theory establishes universal approximation. Its primary focus is the end-to-end learned architecture rather than a controlled comparison between fixed and learned point tokenizers under a common downstream Transformer. Section~\ref{sec:encoder_ablation} isolates that choice by replacing only the fixed Fourier encoder with a learned two-layer raw-coordinate point MLP. \mst{}~\citep{multiset-transformer-2025} instead encodes point multiplicities in its attention weights and uses multiset-specific equivariant blocks and invariant readout (implementation details below). \xpert{}~\citep{kim2024xpert} discretizes persistence diagrams through hard grid binning and processes the resulting patches with a ViT-style Transformer. Persistence images instead form smooth additive surfaces before rasterization. Our grid analysis isolates the resolution cost of a uniform cell-averaging abstraction of hard binning rather than treating all grid-based summaries as equivalent. Other learned vectorizations include the differentiable structure-element layers introduced by~\citet{hofer2017deep} and extended by~\citet{hofer2019learning}, as well as the learned unsupervised measure vectorization of ATOL~\citep{royer2021atol}. Persistent homology and message passing can capture complementary information~\citep{horn2022togl}. We therefore treat raw-input GNNs as task-level references rather than PD-encoder baselines, while joint graph--topology models remain a natural extension. Under the unified ordinary-PD sweep in Table~\ref{tab:graph_benchmarks},~\ours{} achieves the highest mean on every reported dataset among the evaluated Transformer-based PD methods.

\smallskip
\noindent\textbf{Implementation details of Transformer-based PD interfaces.}
Section~\ref{sec:intro} summarizes where the Transformer-based PD methods discussed there introduce diagram-specific design. Here we give the corresponding implementation-level details, covering point encoding, encoder connectivity, token construction, multiplicity handling, readout, and coupled representation hyperparameters.

\smallskip
\noindent\emph{\persformer{}}~\citep{reinauer2021persformer} retains standard scaled dot-product multi-head attention while introducing diagram-specific choices in the surrounding components. Points pass through a trainable position-wise fully connected embedding, homology dimension is supplied as a one-hot input, residual connections are added between self-attention layers to counter the small gradients the authors attribute to closely spaced points, and the readout is attention pooling against a single trainable query. Comparing a point representation directly against~\persformer{} therefore also changes the encoder connectivity and the readout.

\smallskip
\noindent\emph{\xpert{}}~\citep{kim2024xpert} concentrates its design in the input representation. A diagram is normalized, pixelized onto an $R\times R$ grid, and partitioned into $P\times P$ patches, of which the non-empty ones become tokens. Extended diagrams are represented as multi-channel tensors. The resulting tokens are processed by a ViT-style encoder with sinusoidal positional encodings and a CLS readout. This introduces a coupling between the two representation choices. Before empty patches are removed there are $(R/P)^2$ patches, so the patch size must be compatible with the grid resolution.

\smallskip
\noindent\emph{\mst{}}~\citep{multiset-transformer-2025} adapts the Transformer to multisets. A diagram is a set of distinct points with a separate multiplicity vector, and multiplicity-dependent terms are added to the attention weights rather than appended as coordinates. Attention, the equivariant blocks and the invariant readout are consequently all multiset-specific. DBSCAN clustering may optionally be applied first, with the original work reporting results with and without it, and for graph classification each diagram is processed by an independent MST instance before the resulting representations are concatenated.

\smallskip
\noindent\emph{\ours{}} confines its design to the input. 
A point becomes a token through fixed random Fourier features of $(b,d)$ together with a learnable homology-dimension embedding, and the blocks, the attention rule and the CLS readout are unchanged. 
The corresponding encoder-only control is reported in Section~\ref{sec:encoder_ablation}.

\smallskip
\noindent\textbf{Fourier Features and Transformers for Point Sets.}
\citet{tancik2020fourier} showed that mapping coordinates through random Fourier features overcomes the spectral bias of MLPs, enabling learning of high-frequency patterns, a technique widely adopted in NeRF~\citep{mildenhall2020nerf}. We adapt this principle to persistence diagram coordinates. The fixed-backbone ablation in Section~\ref{sec:encoder_ablation} shows that the point-encoder choice can materially affect performance, with the margin varying across datasets. Transformers~\citep{vaswani2017attention} have been adapted for unordered point sets via Set Transformer~\citep{lee2019set} and Point Transformer~\citep{zhao2021point}. Since persistence diagrams are naturally unordered point sets, permutation-equivariant self-attention is a principled choice.

\smallskip 
\noindent\textbf{Concurrent and complementary work.}
A growing thread combines coordinate encodings with topological features, among them bi-continuous functional representations~\citep{persistence-spheres-2026}, custom multiset attention~\citep{multiset-transformer-2025}, measure-based learning~\citep{hacquard2023statistical}, graphcodes~\citep{kerber2024graphcode}, and PiPE~\citep{verma2025pipe}.~\ours{} sits in this thread on the simplest representational end, using fixed random Fourier features on PD coordinates.

\section{Expressivity and Approximation}
\label{app:Expressivity_and_Approximation}

\subsection{Proof of the grid reconstruction lower bound}
\label{app:exp:c3_grid_capacity}

\begin{proof}[Proof of Theorem~\ref{thm:grid_capacity}]
Fix $R\in\mathbb N$ and set $\delta=c/R$. Define
\[
f_1(x)=\cos(2\pi x_1/\delta),
\qquad
f_2(x)=\cos(2\pi x_2/\delta).
\]
Both functions belong to $C(X)$ and satisfy
$\|f_1\|_{\infty,X}=\|f_2\|_{\infty,X}=1$.

Each grid cell has side length $\delta$. For every cell $C_{ij}$,
\[
\int_{C_{ij}} f_1(x)\,dx
=
\delta\int_{(i-1)\delta}^{i\delta}
\cos(2\pi t/\delta)\,dt
=
0,
\]
and the same calculation gives $\int_{C_{ij}} f_2(x)\,dx=0$. Hence
\[
E_{\mathrm{grid},R}(f_1)
=
E_{\mathrm{grid},R}(f_2)
=
0.
\]

Since $R$ is an integer,
\[
\|f_1\|_{L^2(X)}^2
=
\|f_2\|_{L^2(X)}^2
=
\frac{c^2}{2},
\qquad
\langle f_1,f_2\rangle_{L^2(X)}=0.
\]
Therefore $\|f_1-f_2\|_{L^2(X)}=c$.

For any decoder $\mathcal G$, let $g=\mathcal G(0)$. By the triangle inequality,
\[
c
=
\|f_1-f_2\|_{L^2(X)}
\le
\|f_1-g\|_{L^2(X)}
+
\|f_2-g\|_{L^2(X)}.
\]
Thus at least one of the two reconstruction errors is at least $c/2$, and therefore
\[
\sup_{\substack{f\in C(X)\\ \|f\|_{\infty,X}\le1}}
\|f-\mathcal G(E_{\mathrm{grid},R}(f))\|_{L^2(X)}
\ge
\frac{c}{2}.
\]
\end{proof}

\subsection{Proof of universal approximation}
\label{app:exp:c3_realization}

\paragraph{Domain and topology.}
Fix $c>0$ and an integer $q\ge1$. Let
$\mathcal H_c=\{(b,d)\in[0,c]^2\mid b\le d\}$ and
$\mathcal D_c=\bigsqcup_{N=0}^{q}\mathcal H_c^N/\mathfrak S_N$,
with the quotient topology on each component and the disjoint-union
topology on $\mathcal D_c$. The $N=0$ component contains the empty
diagram, and diagonal points are retained. Each component is compact,
so $\mathcal D_c$ is compact.

\paragraph{Architecture.}
Fix an integer $H\ge1$ and choose $d_{\mathrm{model}}\ge6q+10$
divisible by $H$. We use two pre-LayerNorm Transformer layers with
$H$ softmax attention heads of dimension $d_{\mathrm{model}}/H$
and affine--GELU--affine FFNs. The model has an affine input projection,
a trainable CLS token, and a final LayerNorm followed by a scalar
linear head. There is no sequence-position encoding or dropout.
Padding keys are masked, and CLS remains unmasked. Each LayerNorm
has a fixed positive variance constant, unit scale, and zero shift.

\begin{proof}[Proof of Theorem~\ref{thm:ppe_fixed_encoder_universality}]
Let $\Omega\in\mathbb R^{m\times2}$, $m\ge3$, have rows
$\omega_j^\top$, where the $\omega_j$ are sampled i.i.d. from a
nondegenerate Gaussian distribution on $\mathbb R^2$. We use the
paired features
$(\cos(2\pi\omega_j^\top x),\sin(2\pi\omega_j^\top x))$.

Let $A$ consist of the first two rows of $\Omega$. Almost surely,
$A$ is invertible. If the first three Fourier pairs agree at $x$
and $y$, then $v=x-y$ satisfies
\[
v=A^{-1}k,
\qquad
\omega_3^\top A^{-1}k\in\mathbb Z
\]
for some $k\in\mathbb Z^2$. Conditional on an invertible $A$,
$\omega_3^\top A^{-1}k$ is a nondegenerate Gaussian variable
for each nonzero $k$, so it belongs to $\mathbb Z$ with probability
zero. Taking the countable union over
$k\in\mathbb Z^2\setminus\{0\}$ shows that the map given by the first
three Fourier pairs is injective on $\mathbb R^2$ almost surely.

Fix such an $\Omega$. Let $F\colon\mathcal D_c\to\mathbb R$ be
continuous and let $\varepsilon>0$.

\paragraph{Moment representation.}
Set $p=2q+1$. For $x=(b,d)\in\mathbb R^2$, write
$\zeta=b+\mathrm i d$ and define
\[
\phi(x)
=
\bigl(
1,\operatorname{Re}\zeta,\operatorname{Im}\zeta,
\ldots,\operatorname{Re}\zeta^q,\operatorname{Im}\zeta^q
\bigr)
\in\mathbb R^p.
\]
For a diagram $D$, let $M(D)=\sum_{x\in D}\phi(x)$.
The first coordinate gives $N=|D|$. The first $N$ complex power
sums determine $\prod_{j=1}^{N}(t-\zeta_j)$ by Newton's identities.
Thus $M$ determines $D$, including multiplicities. Since $M$ is
continuous and injective on the compact space $\mathcal D_c$,
it has a continuous inverse on its image.

\paragraph{LayerNorm.}
For $\tau>0$, write
\[
\operatorname{LN}_{\tau}(u)
=
\frac{Pu}{\sqrt{\|Pu\|_2^2/d_{\mathrm{model}}+\tau}},
\qquad
P=I-\frac{\mathbf1\mathbf1^\top}{d_{\mathrm{model}}}.
\]
The radial map
$r\mapsto r/\sqrt{r^2/d_{\mathrm{model}}+\tau}$
is strictly increasing for $r\ge0$. Hence equal LayerNorm outputs
imply that the inputs differ by a multiple of $\mathbf1$.
On a set whose elements share one fixed coordinate, this multiple
must be zero. LayerNorm is therefore injective on such a set and,
if the set is compact, has a continuous inverse on its image.

Its Jacobian has operator norm at most $\tau^{-1/2}$, giving
\begin{equation}
\|\operatorname{LN}_{\tau}(u)-\operatorname{LN}_{\tau}(v)\|_2
\le
\frac{\|u-v\|_2}{\sqrt{\tau}}.
\label{eq:ppe_univ_ln_lipschitz}
\end{equation}

Since GELU is continuous and nonpolynomial, affine--GELU--affine
FFNs approximate continuous scalar functions uniformly on compact
sets \citep{leshno1993nonpolynomial}. The vector-valued case follows
by applying the scalar result componentwise after continuous
extension from the compact domain.

\paragraph{First layer.}
Choose the affine input projection so that the first three Fourier
pairs occupy the first six coordinates of each point token.
Set coordinates $2p+1$ and $2p+2$ to $1$ and $-1$, respectively,
and all remaining coordinates to zero. Choose CLS with the same
two fixed coordinates and zeros elsewhere.

The point-token map is continuous and injective. The first six
coordinates have squared norm $3$ for point tokens and vanish
for CLS. Thus CLS is separated from the point tokens.
Their union $E$ is compact, and LayerNorm is injective on $E$.

Choose
$r_0>\max\{1,\|F\|_\infty,
\sup_{x\in\mathcal H_c}\|\phi(x)\|_2\}$.
For $t\in\mathbb R^p$ with $\|t\|_2\le r_0$, define
\[
J(t)
=
\left(
t,-t,1,-1,
\sqrt{r_0^2-\|t\|_2^2},
-\sqrt{r_0^2-\|t\|_2^2},
0_{d_{\mathrm{model}}-2p-4}
\right).
\]
Then
\begin{equation}
\mathbf1^\top J(t)=0,
\qquad
\|J(t)\|_2^2=2r_0^2+2.
\label{eq:ppe_univ_code}
\end{equation}
Since $d_{\mathrm{model}}\ge3p+4$, the last $p$ coordinates
of $J(t)$ vanish.

Set the first attention output to zero. For a point token $u_x$,
prescribe the FFN update $J(\phi(x))-u_x$; for CLS, prescribe
$J(0)-u_{\mathrm{CLS}}$. These updates are continuous functions
of the normalized inputs because LayerNorm has a continuous
inverse on its image. Hence, for every $\eta>0$, the first FFN
can be chosen so that its residual outputs satisfy
\[
\sup_{x\in\mathcal H_c}
\|\widetilde z_x-J(\phi(x))\|_2<\eta,
\qquad
\|\widetilde z_{\mathrm{CLS}}-J(0)\|_2<\eta.
\]

\paragraph{Second layer.}
Let $\tau_a,\tau_f,\tau_o>0$ be the LayerNorm constants for the
second attention branch, the second FFN, and the final output,
respectively. Set
$s_a=\sqrt{(2r_0^2+2)/d_{\mathrm{model}}+\tau_a}$.
By~\eqref{eq:ppe_univ_code},
$\operatorname{LN}_{\tau_a}(J(t))=J(t)/s_a$.

Let $S$ extract the first $p$ coordinates, and let $\iota$ insert
a vector into the last $p$ coordinates. Set the query and key
projections to zero. Choose the value projections as the $H$
consecutive coordinate blocks of the identity matrix, and choose
the output projection as $s_a\iota S$. Set all attention biases
to zero.

For a diagram $D$ with $N=|D|$, each head assigns weight $1/(N+1)$
to every unmasked token. Thus, for the tokens $J(\phi(x))$ and
$J(0)$, the CLS residual after attention is
\begin{equation}
y_D=J(0)+\iota(a_D),
\qquad
a_D=\frac{M(D)}{N+1}.
\label{eq:ppe_univ_cls}
\end{equation}
Since $(a_D)_1=N/(N+1)$,
\[
M(D)=\frac{a_D}{1-(a_D)_1},
\qquad
1-(a_D)_1=\frac{1}{N+1}\ge\frac{1}{q+1}.
\]
Hence $y_D$ determines $D$ continuously.

The set $Y=\{y_D\mid D\in\mathcal D_c\}$ is compact. Attention
leaves the two fixed coordinates of $J(0)$ unchanged, so
$\operatorname{LN}_{\tau_f}$ is injective on $Y$ and has a
continuous inverse on its image.

Choose the final linear head so that the output map is
\[
\mathcal R(v)=s_o[\operatorname{LN}_{\tau_o}(v)]_1,
\qquad
s_o=\sqrt{\frac{2r_0^2+2}{d_{\mathrm{model}}}+\tau_o}.
\]
For the first standard basis vector $e_1\in\mathbb R^p$,
\eqref{eq:ppe_univ_code} gives
$\mathcal R(J(F(D)e_1))=F(D)$.

The map
\[
\operatorname{LN}_{\tau_f}(y_D)
\longmapsto
J(F(D)e_1)-y_D
\]
is continuous on a compact set. Choose the second FFN $f_2$
to approximate it with uniform Euclidean error less than
$\varepsilon\sqrt{\tau_o}/(2s_o)$.
By~\eqref{eq:ppe_univ_ln_lipschitz}, $\mathcal R$ is
$s_o/\sqrt{\tau_o}$-Lipschitz. Therefore
$h(y)=\mathcal R(y+f_2(\operatorname{LN}_{\tau_f}(y)))$
satisfies
\[
\sup_{D\in\mathcal D_c}|h(y_D)-F(D)|
<
\frac{\varepsilon}{2}.
\]

\paragraph{Approximation error.}
Let $\widetilde y_D$ be the CLS residual obtained from the
approximate first-layer outputs. The attention weights remain
uniform because the query and key projections are zero.
Since $S$ and $\iota$ have operator norm one,
\eqref{eq:ppe_univ_ln_lipschitz} gives
\[
\sup_{D\in\mathcal D_c}
\|\widetilde y_D-y_D\|_2
\le
\left(1+\frac{s_a}{\sqrt{\tau_a}}\right)\eta.
\]

With $f_2$ fixed, $h$ is uniformly continuous on a compact
neighborhood of $Y$. Choose $\eta$ sufficiently small, and then
choose the first FFN with this tolerance, so that
$\sup_{D\in\mathcal D_c}|h(\widetilde y_D)-h(y_D)|
<\varepsilon/2$.
Since $T_{\theta,\Omega}(D)=h(\widetilde y_D)$, it follows that
\[
\sup_{D\in\mathcal D_c}
|T_{\theta,\Omega}(D)-F(D)|
<
\varepsilon.
\]
\end{proof}

\paragraph{Multiple homology dimensions.}
Give the finitely many homology labels the discrete topology.
Choose their embeddings so that one unused coordinate of each
initial point token has a distinct value for each label.
Choose translations $a_k\in\mathbb R^2$ such that
$\mathcal H_c+a_k$ are pairwise disjoint, and replace $\phi(x)$
by $\phi_k(x)=\phi(x+a_k)$ for a point with label $k$.
The resulting moment sums determine both the coordinates and
labels, including multiplicities.
Choose
$r_0>\max\{1,\|F\|_\infty,
\max_k\sup_{x\in\mathcal H_c}\|\phi_k(x)\|_2\}$.
The same construction then applies when the total number of
points is at most $q$.

\section{ORBIT5K Noise Sweep}
\label{app:noise_sweep}

We extend the coordinate-encoder ablation of Section~\ref{sec:encoder_ablation} along the coordinate-noise axis on ORBIT5K. The~\ours{} backbone (Transformer encoder, homology-dimension embedding, CLS pooling, linear head, optimizer, and data pipeline) is held fixed. The coordinate encoder alternates between the default fixed Gaussian random-Fourier features (\texttt{model.tokenizer\_family=fourier}) and a learned two-layer MLP point encoder over raw $(b,d)$ coordinates (\texttt{tokenizer\_family=mlp}). Both variants use \texttt{model.encoding\_variant=rff\_plain} and \texttt{model.use\_betti\_embedding=false}. We report mean test accuracy across three training seeds per noise level. Following \citet{adams2017persistence}, we sweep the noise level---the fraction of points replaced by uniform draws over the $(b, d)$ bounding box---across $\{0, 0.05, 0.10, \dots, 0.35\}$.

\begin{figure}[htbp]
\centering
\includegraphics[width=0.6\linewidth]{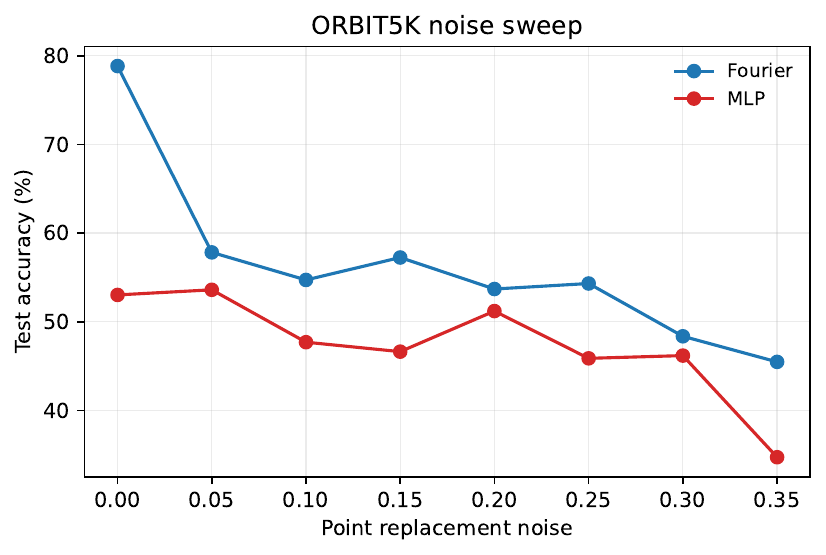}
\caption{\textbf{ORBIT5K coordinate-noise sweep.} Mean test accuracy as a function of point-replacement noise. The fixed Fourier encoder wins every noise level. The gap is largest at zero noise (+25.8~pp) and persists across the sweep. Averaged over the seven noisy settings, Fourier reaches 53.1\% versus the learned MLP's 46.6\%.}
\label{fig:noise_sweep}
\end{figure}

Fixed high-frequency features preserve discriminative structure under coordinate corruption. The learned MLP must induce comparable features and visibly loses ground as the noise budget grows. This sweep uses a fixed split with three seeds per noise level and is reported separately from the 10-fold CV protocol of Table~\ref{tab:graph_benchmarks}.

\section{Sub-cell Donut Benchmark and Tokenization Visualization}
\label{app:donut_subcell}

This appendix provides visual evidence for the grid-collapse mechanism discussed in Section~\ref{sec:donut_subcell}. Under~\xpert{}'s hard $R\times R$ grid binning, plain and sprinkle donuts produce token representations that are visually nearly indistinguishable, even though the two classes are statistically separable on the raw persistence-diagram coordinates.

\paragraph{Construction.} The plain-vs-sprinkle donut point clouds are generated via a three-stage text${\to}$image${\to}$mesh${\to}$point-cloud pipeline. \emph{(i) Text-to-image.} SDXL-Turbo (\url{https://github.com/Stability-AI/generative-models}) renders 2D images from class-conditional prompts (representative samples in Figure~\ref{fig:donut_sdxl_samples}). \emph{(ii) Image-to-3D.} TripoSR (\url{https://github.com/VAST-AI-Research/TripoSR}) lifts each image to a textured triangle mesh. \emph{(iii) Mesh-to-point-cloud.} We take the mesh vertices as the raw point cloud and uniformly randomly subsample to 2048 points, serialized as \texttt{.npy}. We generated 100 candidates per class and manually pruned obvious failure modes---bitten or partially-formed donuts, malformed 3D reconstructions, and other low-quality outputs---yielding 77 plain and 87 sprinkle samples (164 total) for the benchmark.

\begin{figure}[!htbp]
\centering
\includegraphics[width=0.32\linewidth]{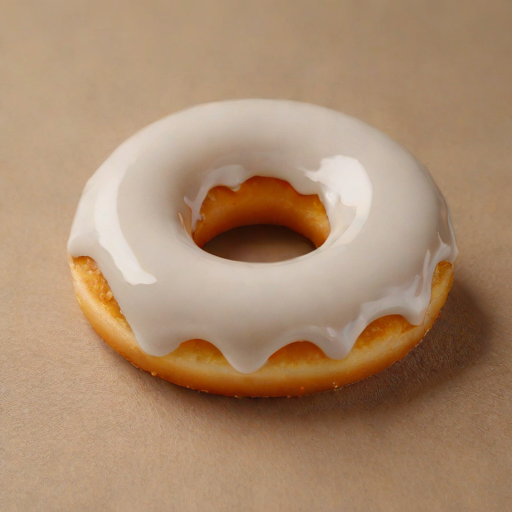}\hspace{0.04\linewidth}\includegraphics[width=0.32\linewidth]{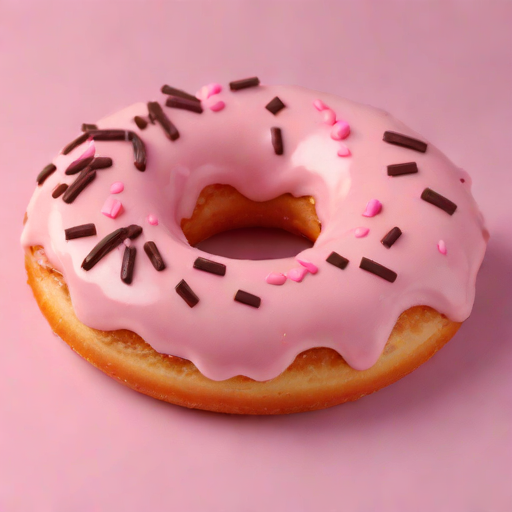}
\caption{\textbf{Representative SDXL-Turbo outputs at stage (i) of the donut-dataset construction pipeline.} Left shows a \emph{plain} donut sample. Right shows a \emph{sprinkle} donut sample. These 2D renderings are subsequently lifted to 3D meshes by TripoSR and uniformly subsampled to 2048-point clouds.}
\label{fig:donut_sdxl_samples}
\end{figure}

\begin{figure}[!htbp]
\centering
\includegraphics[width=0.5\linewidth]{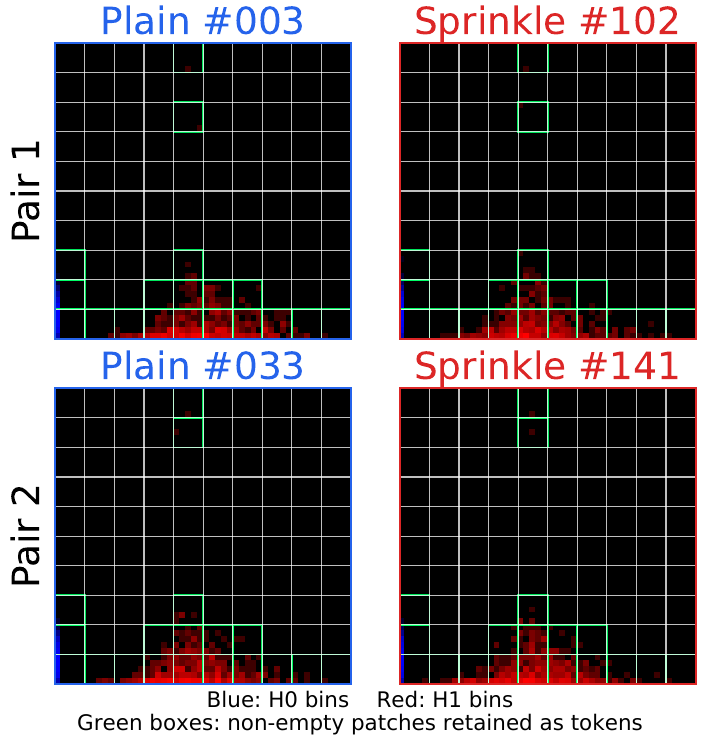}
\caption{\textbf{\xpert{} tokenization of plain (left) versus sprinkle (right) donut persistence diagrams.} Each row pairs a plain-donut sample with a sprinkle-donut sample whose~\xpert{} tokens activate the same grid patches. The hard $R\times R$ grid binning maps every PD point in a cell to the cell centre, so sub-cell coordinate shifts---which carry the plain-vs-sprinkle distinction---are averaged away. The two classes therefore become indistinguishable at the tokenization stage, the empirical analogue of the grid-collapse mechanism formalized in Theorem~\ref{thm:grid_capacity}.}
\label{fig:xpert_donut_tokenization}
\end{figure}

\begin{table}[!htbp]
\centering
\caption{\textbf{Plain-vs-Sprinkle Donut classification.} Accuracy (\%) over ten folds and five seeds per fold. For each method, we report the configuration selected by validation accuracy during random hyperparameter search. Gray shading and \textbf{bold} mark the highest mean in each column, while \underline{underline} marks the second-highest.}
\label{tab:donut_classification}
\footnotesize \setlength{\tabcolsep}{6pt}
\begin{tabular}{lcc}
\toprule
\textbf{Method} & \textbf{Val Acc (\%)} & \textbf{Test Acc (\%)} \\
\midrule
\mst{} & \underline{64.12 $\pm$7.83} & 50.91 $\pm$10.34 \\
\persformer{} & 55.18 $\pm$9.18 & 52.32 $\pm$11.19 \\
\xpert{} & 62.82 $\pm$9.80 & \underline{54.11 $\pm$12.62} \\
\ours{} (Ours) & \cellcolor{gray!15}\textbf{77.29 $\pm$11.71} & \cellcolor{gray!15}\textbf{71.97 $\pm$11.36} \\
\bottomrule
\end{tabular}
\end{table}

\FloatBarrier
\section{Architecture Details}
\label{app:architecture}

\begin{table}[!htbp]
\centering
\caption{\textbf{\ours{} Transformer hyperparameters.} Architecture and training defaults for the unified ORBIT5K and TU experiments. Per-dataset learning rate, weight decay, and dropout are tuned via random grid search on the validation split; batch size, warmup, total epochs, and early-stopping patience are fixed per dataset family (full protocol in~\Aref{app:grid_search}).}
\label{tab:hyperparams}
\footnotesize \setlength{\tabcolsep}{4pt}
\resizebox{\linewidth}{!}{%
\begin{tabular}{ll}
\toprule
\textbf{Component} & \textbf{Configuration} \\
\midrule
\multicolumn{2}{l}{\emph{Fourier Feature Encoder}} \\
\quad Number of frequencies ($m$) & 128 \\
\quad Frequency scale ($\sigma$) & 10.0 \\
\quad Output dimension & 256 ($m$ cosine and $m$ sine components) \\
\midrule
\multicolumn{2}{l}{\emph{Homology-Dimension Embedding}} \\
\quad Number of dimensions & 2 ($H_0$, $H_1$) \\
\quad Embedding dimension & 128 ($=d_{\mathrm{model}}$) \\
\midrule
\multicolumn{2}{l}{\emph{Transformer Encoder}} \\
\quad Model dimension ($d_{\mathrm{model}}$) & 128 \\
\quad Number of heads & 4 \\
\quad Number of layers & 3 \\
\quad Feedforward dimension & 256 \\
\quad Dropout & 0.1 \\
\quad Normalization & pre-LayerNorm \\
\quad Activation & GELU \\
\midrule
\multicolumn{2}{l}{\emph{Training}} \\
\quad Optimizer & AdamW \\
\quad Loss & Cross-entropy \\
\quad LR scheduler & Linear warmup + cosine decay \\
\quad Gradient clip & 1.0 \\
\quad Precision & fp32 \\
\quad Evaluation protocol & Stratified 10-fold CV, inner val split ($\sim$80/10/10 per fold) \\
\quad Per-dataset lr / wd / dropout & Tuned via random grid search \\
\quad Per-dataset batch / warmup / epochs / patience & Fixed per dataset family \\
\bottomrule
\end{tabular}%
}
\end{table}

\textbf{Larger-scale frozen-transfer settings.} Both larger-scale benchmarks use seeds 101--105, 100 training epochs, early-stopping patience 30, AdamW with learning rate $10^{-3}$ and zero weight decay, gradient clipping at 1.0, and fp32 precision. The training and evaluation batch sizes are 32 and 64 for ogbg-molhiv, and 16 and 32 for MCB-C. The \ours{} encoder retains 128 Fourier frequencies at scale 10, model width 128, four attention heads, three Transformer layers, feedforward width 256, and dropout 0.1. MCB-C uses component embeddings for $H_0$, $H_1$, and $H_2$.

\paragraph{Compute resources.}
Experiments were executed as single-GPU jobs on available compatible accelerators, including NVIDIA A100, RTX A6000, RTX PRO 6000, RTX 3090, and RTX 4090 GPUs. The full experimental campaign therefore used heterogeneous hardware, depending on resource availability. For each (method, dataset) pair, the complete random hyperparameter search comprises 30 configurations $\times$ 10 folds $\times$ 5 random seeds, or $1{,}500$ training runs (Appendix~\ref{app:grid_search}). To provide concrete reference measurements on 24-GB consumer GPUs,  Table~\ref{tab:compute_per_dataset} reports logged wall-clock times for the final selected~\ours{} configuration executed on a single RTX 3090 or RTX 4090. Each entry covers 5 random seeds $\times$ 10-fold cross-validation and excludes the preceding hyperparameter search, preliminary experiments, and configurations not included in the final paper. Because the entries were collected on two different GPU models rather than under a controlled hardware benchmark, they should be interpreted as representative reproduction times rather than direct cross-dataset runtime comparisons.

\begin{table}[h]
\centering
\caption{\textbf{Representative logged wall-clock times for the final selected~\ours{} configuration.} Each entry covers 5 random seeds $\times$ 10-fold cross-validation and was measured on a single RTX 3090 or RTX 4090 GPU (24~GB). These measurements demonstrate reproduction on 24 GB consumer hardware.}
\label{tab:compute_per_dataset}
\footnotesize \setlength{\tabcolsep}{6pt}
\begin{tabular}{lr}
\toprule
\textbf{Dataset} & \textbf{Wall-clock time} \\
\midrule
MUTAG & 3m 28s \\
DHFR & 7m 30s \\
COX2 & 7m 45s \\
IMDB-B & 9m 46s \\
PROTEINS & 22m 4s \\
IMDB-M & 24m 50s \\
NCI1 & 1h 35m 39s \\
Donut (sub-cell) & 1h 46m 35s \\
NCI109 & 1h 50m 44s \\
ORBIT5K & 1d 12h 47s \\
\bottomrule
\end{tabular}
\end{table}

\begin{table}[h]
\centering
\caption{\textbf{MCB-C token-cap sensitivity.} The cap applies to the computed persistence diagram, bounding the number of points taken from each homology dimension and keeping those of largest lifetime $d-b$. MCB-C uses $H_0$, $H_1$, and $H_2$, so at cap $C$ a diagram contributes at most $3C$ points. Values are balanced accuracy over 50 seed-runs.}
\label{tab:mcb_cap_sensitivity}
\footnotesize \setlength{\tabcolsep}{7pt}
\begin{tabular}{@{}lc@{}}
\toprule
\makecell{\textbf{Cap per homology}\\[-1pt]\textbf{component}} & \textbf{Balanced accuracy} \\
\midrule
512  &  65.82{\tiny$\pm$3.59} \\
1,024 (primary)  &  65.79{\tiny$\pm$4.27} \\
2,048  &  65.10{\tiny$\pm$4.07} \\
\bottomrule
\end{tabular}

\end{table}

\FloatBarrier
\section{Random Grid Search Configuration}
\label{app:grid_search}

All six methods follow the same random-search protocol. For each (method, dataset) we draw 30 random hyperparameter configurations and evaluate each under the full stratified 10-fold CV protocol with 5 fixed seeds; the optimization metric is the cross-fold cross-seed mean validation accuracy, and selection maximizes this mean. The total cost per (method, dataset) is therefore $30 \times 10 \times 5 = 1{,}500$ training runs. The shared optimization grid (\Tref{tab:grid_search_shared}) is identical across methods for an apples-to-apples comparison; per-method architecture grids (\Tref{tab:grid_search_arch}) specify the evaluated search spaces. Per-dataset training defaults (batch, warmup, epoch limit, early-stopping patience) are fixed per dataset family (\Tref{tab:grid_search_defaults}) and are not part of the search. The expanded PersLay variant comparison is reported separately in Table~\ref{tab:perslay_variants}.

\begin{table}[htbp]
\centering
\caption{\textbf{Shared optimization-hyperparameter grid.} Identical across all six methods.}
\label{tab:grid_search_shared}
\small
\begin{tabular}{lll}
\toprule
\textbf{Hyperparameter} & \textbf{Range} & \textbf{Distribution} \\
\midrule
learning rate  & $[10^{-5},\, 3\times 10^{-3}]$ & log-uniform \\
weight decay   & $[10^{-6},\, 10^{-2}]$ & log-uniform \\
dropout        & $[0,\, 0.5]$ & uniform \\
\bottomrule
\end{tabular}
\end{table}

\begin{table}[htbp]
\centering
\caption{\textbf{Per-method architecture-hyperparameter grids.} Method-specific search space drawn under the shared random-search protocol.}
\label{tab:grid_search_arch}
\small
\begin{threeparttable}
\begin{tabular}{lll}
\toprule
\textbf{Method} & \textbf{Hyperparameter} & \textbf{Values} \\
\midrule
\multirow{6}{*}{\ours{}\tnote{a}}
  & \texttt{d\_model}              & $\{64,\, 128,\, 192,\, 256\}$ \\
  & \texttt{nhead}                 & $\{2,\, 4,\, 8\}$ \\
  & \texttt{num\_layers}           & $\{2,\, 3,\, 4,\, 5\}$ \\
  & \texttt{num\_frequencies}      & $\{64,\, 128,\, 256,\, 384\}$ \\
  & \texttt{ff\_scale}             & $[1,\, 30]$, log-uniform \\
  & \texttt{use\_betti\_embedding} & $\{\text{true},\, \text{false}\}$ \\
\midrule
\multirow{6}{*}{\mst{}} & \texttt{hidden\_size} & $\{64,\, 96,\, 128\}$ \\
  & \texttt{num\_heads}               & $\{2,\, 4\}$ \\
  & \texttt{num\_equivariant\_layers} & $\{1,\, 2,\, 3\}$ \\
  & \texttt{num\_inducing}            & $\{2,\, 8,\, 16\}$ \\
  & \texttt{num\_pool\_queries}       & $\{1,\, 4,\, 8\}$ \\
  & \texttt{classifier\_dropout}     & $[0,\, 0.5]$, uniform \\
\midrule
\multirow{6}{*}{\persformer{}}
  & \texttt{hidden\_size}            & $\{64,\, 96,\, 128,\, 192,\, 256\}$ \\
  & \texttt{num\_attention\_layers}  & $\{2,\, 3,\, 4,\, 5,\, 6\}$ \\
  & \texttt{num\_attention\_heads}   & $\{2,\, 4,\, 8\}$ \\
  & \texttt{intermediate\_size}      & $\{128,\, 256,\, 384,\, 512,\, 768,\, 1024\}$ \\
  & \texttt{pooling}                 & $\{\text{attention},\, \text{attention\_sum}\}$ \\
  & \texttt{classifier\_dropout}     & $[0,\, 0.5]$, uniform \\
\midrule
\multirow{7}{*}{\xpert{}}
  & \texttt{grid\_size}              & $\{40,\, 50\}$ \\
  & \texttt{patch\_size}             & $\{2,\, 4,\, 5,\, 8,\, 10\}$ \\
  & \texttt{embed\_dim}              & $\{64,\, 96,\, 128,\, 192,\, 256\}$ \\
  & \texttt{depth}                   & $\{2,\, 3,\, 4,\, 5,\, 6\}$ \\
  & \texttt{num\_heads}              & $\{2,\, 4,\, 8\}$ \\
  & \texttt{mlp\_ratio}              & $\{2,\, 3,\, 4\}$ \\
  & \texttt{classifier\_dropout}     & $[0,\, 0.5]$, uniform \\
\midrule
\multirow{6}{*}{\perslay{}\tnote{b}}
  & \texttt{pweight\_power}          & $\{0.5,\, 1,\, 2\}$ \\
  & \texttt{pweight\_grid\_size}     & $\{10,\, 20,\, 30\}$ \\
  & \texttt{image\_size}             & $\{16,\, 20,\, 24,\, 32\}$ \\
  & \texttt{image\_sigma}            & $[0.02,\, 0.3]$, log-uniform \\
  & \texttt{image\_num\_filters}     & $\{8,\, 10,\, 16,\, 24,\, 32\}$ \\
  & \texttt{image\_filter\_size}     & $\{2,\, 3,\, 5\}$ \\
\midrule
\multirow{7}{*}{\persimage{}+MLP}
  & \texttt{resolution}              & $\{16,\, 20,\, 24,\, 32\}$ \\
  & \texttt{sigma}                   & $[0.02,\, 0.3]$, log-uniform \\
  & \texttt{weight}                  & $\{\text{none},\, \text{persistence},\, \text{linear\_ramp}\}$ \\
  & \texttt{ramp\_end}               & $\{0.5,\, 1,\, 2\}$ \\
  & \texttt{normalize}               & $\{\text{true},\, \text{false}\}$ \\
  & \texttt{pixel\_integration}      & $\{\text{true},\, \text{false}\}$ \\
  & \texttt{classifier\_hidden\_dim} & $\{0,\, 64,\, 128,\, 256,\, 384,\, 512\}$ \\
\bottomrule
\end{tabular}
\begin{tablenotes}\footnotesize
\item[a]\ours{} fixes \texttt{dim\_feedforward}$=2\,d_{\mathrm{model}}$ (not swept independently).
\item[b]The architecture grid above describes the Image configuration in the original unified sweep, where \texttt{layer\_type=image}, \texttt{pweight=grid}, \texttt{perm\_op=sum}, and \texttt{normalize\_inputs=true} are fixed. The expanded comparison additionally evaluates Line/Permutation-Equivariant, Landscape, and Exponential variants as reported in Table~\ref{tab:perslay_variants}.
\end{tablenotes}
\end{threeparttable}
\end{table}

\begin{table}[htbp]
\centering
\caption{\textbf{Fixed per-dataset training defaults.} Selected once per dataset family and not part of the search; same defaults apply to all methods on a given dataset family.}
\label{tab:grid_search_defaults}
\footnotesize \setlength{\tabcolsep}{4pt}
\resizebox{\linewidth}{!}{%
\begin{tabular}{lcccc}
\toprule
\textbf{Dataset family} & \textbf{Batch} & \textbf{Warmup epochs} & \textbf{Total epochs} & \textbf{ES patience} \\
\midrule
TU graphs (MUTAG, PROTEINS, COX2, DHFR, IMDB-B/M, NCI1/109) & 64 & 50 & 300 & 20 \\
ORBIT5K              & 16 & 50 & 300 & 20 \\
Synthetic donut      & 32 & 5  & 400 & 60 \\
\bottomrule
\end{tabular}%
}
\end{table}

\FloatBarrier
\section{Prior-paper benchmark numbers}
\label{app:prior_paper_numbers}

For reference, Table~\ref{tab:prior_paper_numbers} records each baseline's accuracy as published in its original paper. These numbers retain each paper's original pipeline---persistence variant, splits, and per-method tuning---and are not directly comparable to the unified reproduction in Table~\ref{tab:graph_benchmarks}.

\begin{table}[htbp]
\centering
\caption{\textbf{Prior-paper accuracy (\%).} As reported in each method's original publication, using each paper's original pipeline. The \mst{} row uses ordinary persistence diagrams without clustering~\citep{multiset-transformer-2025}; dashes denote unreported results.}\label{tab:prior_paper_numbers}
\scriptsize \setlength{\tabcolsep}{3.0pt}
\resizebox{\linewidth}{!}{%
\begin{tabular}{lccccccccc}
\toprule
\textbf{Method} & \textbf{ORBIT5K} & \textbf{MUTAG} & \textbf{PROTEINS} & \textbf{COX2} & \textbf{DHFR} & \textbf{IMDB-B} & \textbf{IMDB-M} & \textbf{NCI1} & \textbf{NCI109} \\
\midrule
\quad~\perslay{} & 87.7{\tiny$\pm$1.0} & 89.8 & 74.8 & 80.9 & 80.3 & 71.2 & 48.8 & 73.5 & 69.5 \\
\quad~\persformer{} & 91.2{\tiny$\pm$0.8} & 89.9{\tiny$\pm$2.1} & --- & --- & --- & --- & --- & --- & --- \\
\quad~\xpert{} & 87.0{\tiny$\pm$0.7} & 91.0{\tiny$\pm$5.2} & 75.7{\tiny$\pm$3.8} & 84.4{\tiny$\pm$3.5} & 81.9{\tiny$\pm$2.9} & 72.6{\tiny$\pm$3.4} & 50.0{\tiny$\pm$2.0} & --- & --- \\
\quad~\mst{} (Ord., no clustering) & --- & 89.25{\tiny$\pm$0.77} & 72.92{\tiny$\pm$0.47} & 81.20{\tiny$\pm$1.04} & 74.12{\tiny$\pm$0.37} & 70.76{\tiny$\pm$0.47} & 45.44{\tiny$\pm$0.35} & 69.12{\tiny$\pm$0.17} & 66.38{\tiny$\pm$0.40} \\
\bottomrule
\end{tabular}
}
\vspace{1.2mm}
\begin{minipage}{\linewidth}
\normalsize\raggedright In a separate protocol-sensitivity control of the \xpert{} reproduction, applying the released \xpert{} choices of training on all non-test samples and reporting the maximum test accuracy across epochs raises the nine-dataset macro from 70.77 under validation-only selection to 77.07, with the inputs, folds, seeds, and training settings otherwise fixed between the two protocols. This control is separate from the method-level sweep selections reported in Table~\ref{tab:graph_benchmarks}. Because the released choices select the epoch on the test fold, all unified comparisons in the main text retain validation-only selection.
\end{minipage}
\end{table}

\FloatBarrier
\section{License and Asset Provenance}
\label{app:licenses}

We use the following existing assets, all cited at first use in the main text. URLs and license names below reflect the public repositories at the time of submission.

\smallskip 
\noindent\textbf{Datasets.}
\textbf{TU Dortmund graph collection}~\citep{morris2020tudataset}: MUTAG, PROTEINS, COX2, DHFR, IMDB-B, IMDB-M, NCI1, NCI109. Available at \url{https://chrsmrrs.github.io/datasets/}. The collection landing page does not specify a unified license; sub-datasets originate from heterogeneous publications and are conventionally redistributed for academic research per those publications' terms.

\textbf{ORBIT5K} (5{,}000 linked twist map orbits, 5 classes): generated by the script released with the PersLay paper~\citep{carriere2020perslay}, available at \url{https://github.com/MathieuCarriere/perslay}; the dataset's distribution license follows the PersLay repository's MIT License.

\textbf{ogbg-molhiv}~\citep{hu2020ogb} is distributed through the Open Graph Benchmark at \url{https://ogb.stanford.edu/docs/graphprop/}. The OGB documentation states an MIT License for ogbg-molhiv.

\textbf{MCB-C} is a filtered MCB-A subset released with the Flood Complex study~\citep{graf2025flood} and distributed through \url{https://github.com/plus-rkwitt/flooder}.

\textbf{Plain-vs-Sprinkle Donut benchmark} (newly introduced in Section~\ref{sec:donut_subcell}): constructed by the authors for this paper using SDXL-Turbo and TripoSR (see paragraph below). The point clouds, persistence-diagram computation script, train/validation/test split files, and a complete EDA report will be released alongside our code repository under a CC-BY 4.0 license upon publication.

\smallskip 
\noindent\textbf{Software libraries.}
\textbf{GUDHI}~\citep{gudhi2015}: persistence-diagram computation via Alpha complex and sublevel filtrations. \url{https://gudhi.inria.fr/licensing/}; the core library is released under the MIT License, while a few modules are effectively GPLv3 through dependencies on CGAL, Miniball, or PyKeOps---users employing those modules should consult per-module headers for the effective license. \textbf{PyTorch}: Transformer encoder, optimizer, and training pipeline. \url{https://pytorch.org/}, BSD-3-Clause. \textbf{scikit-learn}: classifier head for the~\persimage{}+MLP baseline. \url{https://scikit-learn.org/}, BSD-3-Clause.

\smallskip 
\noindent\textbf{Baseline encoder implementations.}
We reproduce the three primary Transformer-based baselines under our unified protocol: \xpert{}, \persformer{}, and \mst{}. We additionally evaluate \perslay{} and \persimage{}+MLP as broader PD comparisons. \textbf{\persimage{}+MLP}~\citep{adams2017persistence} is re-implemented from the original paper's specification with a scikit-learn MLP classifier. \textbf{\perslay{}}~\citep{carriere2020perslay} is reproduced from the authors' code at \url{https://github.com/MathieuCarriere/perslay} (MIT License). \textbf{\persformer{}}~\citep{reinauer2021persformer} is reproduced from the giotto-deep code at \url{https://github.com/giotto-ai/giotto-deep}, released under the GNU Affero General Public License v3.0 (AGPL-3.0); we use it only for evaluation and comparison and do not redistribute derivative code. \textbf{\xpert{}}~\citep{kim2024xpert} is reproduced from the authors' code at \url{https://github.com/sehunfromdaegu/xpert} (MIT License). \textbf{\mst{}}~\citep{multiset-transformer-2025} adapts multiplicity-aware attention and readout from the authors' implementation at \url{https://github.com/minghuax/MST}; optional clustering is not used in our comparison.

\smallskip
\noindent\textbf{Generative tools used for donut dataset construction.}
\textbf{SDXL-Turbo} (Stability-AI \texttt{generative-models}, \url{https://github.com/Stability-AI/generative-models}) is used at dataset-construction time only, for the text-to-image stage, under the Stability AI Non-Commercial Research Community License Agreement. \textbf{TripoSR} (VAST-AI-Research \texttt{TripoSR}, \url{https://github.com/VAST-AI-Research/TripoSR}) is used for the image-to-3D stage under the MIT License. Neither model is shipped with our released code or used at training/inference time; only the resulting \texttt{.npy} point clouds will be released, under CC-BY 4.0.


\newpage

\begin{thebibliography}{33}

\bibitem[Adams et~al.(2017)]{adams2017persistence}
Adams, H., Emerson, T., Kirby, M., Neville, R., Peterson, C., Shipman, P., Chepushtanova, S., Hanson, E., Motta, F., and Ziegelmeier, L.
\newblock Persistence images: A stable vector representation of persistent homology.
\newblock~\emph{Journal of Machine Learning Research}, 18(8):1--35, 2017.

\bibitem[Barnes et~al.(2021)]{barnes2021comparative}
Barnes, D., Polanco, L., and Perea, J.~A.
\newblock A comparative study of machine learning methods for persistence diagrams.
\newblock \emph{Frontiers in Artificial Intelligence}, 4:681174, 2021.

\bibitem[Bubenik(2015)]{bubenik2015statistical}
Bubenik, P.
\newblock Statistical topological data analysis using persistence landscapes.
\newblock~\emph{Journal of Machine Learning Research}, 16(3):77--102, 2015.

\bibitem[Cang \& Wei(2017)]{cangwei2017topologynet}
Cang, Z. and Wei, G.-W.
\newblock TopologyNet: Topology based deep convolutional and multi-task neural networks for biomolecular property predictions.
\newblock~\emph{PLoS Computational Biology}, 13(7):e1005690, 2017.

\bibitem[Cang et~al.(2018)]{cangwuwei2018drugdesign}
Cang, Z., Mu, L., and Wei, G.-W.
\newblock Representability of algebraic topology for biomolecules in machine learning based scoring and virtual screening.
\newblock~\emph{PLoS Computational Biology}, 14(1):e1005929, 2018.

\bibitem[Carlsson(2009)]{carlsson2009topology}
Carlsson, G.
\newblock Topology and data.
\newblock~\emph{Bulletin of the American Mathematical Society}, 46(2):255--308, 2009.

\bibitem[Carri{\`e}re et~al.(2017)]{carriere2017sliced}
Carri{\`e}re, M., Cuturi, M., and Oudot, S.
\newblock Sliced Wasserstein kernel for persistence diagrams.
\newblock In~\emph{International Conference on Machine Learning}, pp. 664--673, 2017.

\bibitem[Carri{\`e}re et~al.(2020)]{carriere2020perslay}
Carri{\`e}re, M., Chazal, F., Ike, Y., Lacombe, T., Royer, M., and Umeda, Y.
\newblock PersLay: A neural network layer for persistence diagrams and new graph topological signatures.
\newblock In~\emph{International Conference on Artificial Intelligence and Statistics}, pp. 2786--2796, 2020.

\bibitem[Castilla Rieso et~al.(2026)]{castillarieso2026standards}
Castilla Rieso, I., Peltier, S., and Carr{\'e}, P.
\newblock Position: Topological Machine Learning Cannot Progress without Experimental Standards.
\newblock In~\emph{Proceedings of the 43rd International Conference on Machine Learning}, 2026.

\bibitem[Chazal et~al.(2015)]{chazal2015silhouette}
Chazal, F., Fasy, B.~T., Lecci, F., Rinaldo, A., and Wasserman, L.
\newblock Stochastic convergence of persistence landscapes and silhouettes.
\newblock~\emph{Journal of Computational Geometry}, 6(2):140--161, 2015.

\bibitem[Devlin et~al.(2019)]{devlin2019bert}
Devlin, J., Chang, M.-W., Lee, K., and Toutanova, K.
\newblock BERT: Pre-training of deep bidirectional transformers for language understanding.
\newblock In~\emph{NAACL-HLT}, 2019.


\bibitem[Gameiro et~al.(2015)]{gameiro2015protein}
Gameiro, M., Hiraoka, Y., Izumi, S., Kramar, M., Mischaikow, K., and Nanda, V.
\newblock A topological measurement of protein compressibility.
\newblock~\emph{Japan Journal of Industrial and Applied Mathematics}, 32(1):1--17, 2015.

\bibitem[Ghahremani et~al.(2018)]{ghahremani2018alteration}
Ghahremani, M., Yoo, J., Chung, S.~J., Yoo, K., Ye, J.~C., and Jeong, Y.
\newblock Alteration in the local and global functional connectivity of resting state networks in {Parkinson}'s disease.
\newblock~\emph{Journal of Movement Disorders}, 11(1):13--23, 2018.

\bibitem[Graf et~al.(2025)]{graf2025flood}
Graf, F., Pellizzoni, P., Uray, M., Huber, S., and Kwitt, R.
\newblock The Flood Complex: Large-scale persistent homology on millions of points.
\newblock In~\emph{Advances in Neural Information Processing Systems}, volume~38, 2025.

\bibitem[Hacquard et~al.(2023)]{hacquard2023statistical}
Hacquard, O., Blanchard, G., and Levrard, C.
\newblock Statistical learning on measures: an application to persistence diagrams.
\newblock~\emph{arXiv preprint arXiv:2303.08456}, 2023.

\bibitem[Hiraoka et~al.(2016)]{hiraoka2016amorphous}
Hiraoka, Y., Nakamura, T., Hirata, A., Escolar, E.~G., Matsue, K., and Nishiura, Y.
\newblock Hierarchical structures of amorphous solids characterized by persistent homology.
\newblock~\emph{Proceedings of the National Academy of Sciences}, 113(26):7035--7040, 2016.

\bibitem[Hofer et~al.(2017)]{hofer2017deep}
Hofer, C., Kwitt, R., Niethammer, M., and Uhl, A.
\newblock Deep learning with topological signatures.
\newblock In~\emph{Advances in Neural Information Processing Systems}, volume~30, 2017.

\bibitem[Hofer et~al.(2019)]{hofer2019learning}
Hofer, C., Kwitt, R., and Niethammer, M.
\newblock Learning representations of persistence barcodes.
\newblock~\emph{Journal of Machine Learning Research}, 20(126):1--45, 2019.

\bibitem[Horn et~al.(2022)]{horn2022togl}
Horn, M., De~Brouwer, E., Moor, M., Moreau, Y., Rieck, B., and Borgwardt, K.
\newblock Topological graph neural networks.
\newblock In~\emph{International Conference on Learning Representations (ICLR)}, 2022.



\bibitem[Hu et~al.(2020)]{hu2020ogb}
Hu, W., Fey, M., Zitnik, M., Dong, Y., Ren, H., Liu, B., Catasta, M., and Leskovec, J.
\newblock Open Graph Benchmark: Datasets for machine learning on graphs.
\newblock In~\emph{Advances in Neural Information Processing Systems}, volume~33, 2020.

\bibitem[Kerber \& Russold(2024)]{kerber2024graphcode}
Kerber, M. and Russold, F.
\newblock Graphcode: Learning from multiparameter persistent homology using graph neural networks.
\newblock In~\emph{Advances in Neural Information Processing Systems}, 2024.

\bibitem[Khasawneh \& Munch(2016)]{khasawnehmunch2016chatter}
Khasawneh, F.~A. and Munch, E.
\newblock Chatter detection in turning using persistent homology.
\newblock~\emph{Mechanical Systems and Signal Processing}, 70--71:527--541, 2016.

\bibitem[Kim et~al.(2020)]{kim2020pllay}
Kim, K., Kim, J., Zaheer, M., Kim, J., Chazal, F., and Wasserman, L.
\newblock PLLay: Efficient topological layer based on persistent landscapes.
\newblock In~\emph{Advances in Neural Information Processing Systems}, 2020.

\bibitem[Kim(2024)]{kim2024xpert}
Kim, S.
\newblock xPerT: Extended Persistence Transformer.
\newblock~\emph{arXiv preprint arXiv:2410.14193}, 2024.

\bibitem[Kipf \& Welling(2017)]{kipf2017gcn}
Kipf, T.~N. and Welling, M.
\newblock Semi-supervised classification with graph convolutional networks.
\newblock In~\emph{International Conference on Learning Representations}, 2017.

\bibitem[Kovacev-Nikolic et~al.(2016)]{kovacevnikolic2016protein}
Kovacev-Nikolic, V., Bubenik, P., Nikoli{\'c}, D., and Heo, G.
\newblock Using persistent homology and dynamical distances to analyze protein binding.
\newblock~\emph{Statistical Applications in Genetics and Molecular Biology}, 15(1):19--38, 2016.

\bibitem[Kusano et~al.(2016)]{kusano2016pwg}
Kusano, G., Hiraoka, Y., and Fukumizu, K.
\newblock Persistence weighted Gaussian kernel for topological data analysis.
\newblock In~\emph{International Conference on Machine Learning}, pp. 2004--2013, 2016.

\bibitem[Lee et~al.(2019)]{lee2019set}
Lee, J., Lee, Y., Kim, J., Kosiorek, A., Choi, S., and Teh, Y.~W.
\newblock Set transformer: A framework for attention-based permutation-invariant neural networks.
\newblock In~\emph{International Conference on Machine Learning}, pp. 3744--3753, 2019.

\bibitem[Leshno et~al.(1993)]{leshno1993nonpolynomial}
Leshno, M., Lin, V.~Y., Pinkus, A., and Schocken, S.
\newblock Multilayer feedforward networks with a nonpolynomial activation function can approximate any function.
\newblock~\emph{Neural Networks}, 6(6):861--867, 1993.

\bibitem[Loshchilov \& Hutter(2019)]{loshchilov2019decoupled}
Loshchilov, I. and Hutter, F.
\newblock Decoupled weight decay regularization.
\newblock In~\emph{International Conference on Learning Representations}, 2019.

\bibitem[Mildenhall et~al.(2020)]{mildenhall2020nerf}
Mildenhall, B., Srinivasan, P.~P., Tancik, M., Barron, J.~T., Ramamoorthi, R., and Ng, R.
\newblock NeRF: Representing scenes as neural radiance fields for view synthesis.
\newblock In~\emph{European Conference on Computer Vision}, pp. 405--421, 2020.

\bibitem[Morris et~al.(2020)]{morris2020tudataset}
Morris, C., Kriege, N.~M., Bause, F., Kersting, K., Mutzel, P., and Neumann, M.
\newblock TUDataset: A collection of benchmark datasets for learning with graphs.
\newblock In~\emph{ICML Workshop on Graph Representation Learning and Beyond}, 2020.

\bibitem[Nguyen et~al.(2019)]{nguyen2019d3r}
Nguyen, D.~D., Cang, Z., Wu, K., Wang, M., Cao, Y., and Wei, G.-W.
\newblock Mathematical deep learning for pose and binding affinity prediction and ranking in {D3R} {G}rand {C}hallenges.
\newblock~\emph{Journal of Computer-Aided Molecular Design}, 33(1):71--82, 2019.

\bibitem[Pegoraro(2026)]{persistence-spheres-2026}
Pegoraro, M.
\newblock Persistence Spheres: Bi-continuous Representations of Persistence Diagrams.
\newblock In~\emph{International Conference on Learning Representations (ICLR)}, 2026.

\bibitem[Perea \& Harer(2015)]{pereaharer2015sliding}
Perea, J.~A. and Harer, J.
\newblock Sliding windows and persistence: An application of topological methods to signal analysis.
\newblock~\emph{Foundations of Computational Mathematics}, 15(3):799--838, 2015.

\bibitem[Petri et~al.(2014)]{petri2014homological}
Petri, G., Expert, P., Turkheimer, F., Carhart-Harris, R., Nutt, D., Hellyer, P.~J., and Vaccarino, F.
\newblock Homological scaffolds of brain functional networks.
\newblock~\emph{Journal of the Royal Society Interface}, 11(101):20140873, 2014.

\bibitem[Qi et~al.(2017)]{qi2017pointnet}
Qi, C.~R., Su, H., Mo, K., and Guibas, L.~J.
\newblock PointNet: Deep learning on point sets for 3D classification and segmentation.
\newblock In~\emph{IEEE Conference on Computer Vision and Pattern Recognition}, pp. 652--660, 2017.

\bibitem[Reinauer et~al.(2021)]{reinauer2021persformer}
Reinauer, R., Caorsi, M., and Berkouk, N.
\newblock Persformer: A transformer architecture for topological machine learning.
\newblock~\emph{arXiv preprint arXiv:2112.15210}, 2021.

\bibitem[Reininghaus et~al.(2015)]{reininghaus2015stable}
Reininghaus, J., Huber, S., Bauer, U., and Kwitt, R.
\newblock A stable multi-scale kernel for topological machine learning.
\newblock In~\emph{IEEE Conference on Computer Vision and Pattern Recognition}, pp. 4741--4748, 2015.

\bibitem[Royer et~al.(2021)]{royer2021atol}
Royer, M., Chazal, F., Levrard, C., Umeda, Y., and Ike, Y.
\newblock ATOL: Measure vectorization for automatic topologically-oriented learning.
\newblock In~\emph{International Conference on Artificial Intelligence and Statistics}, pp. 1000--1008, 2021.

\bibitem[Saadatfar et~al.(2017)]{saadatfar2017granular}
Saadatfar, M., Takeuchi, H., Robins, V., Francois, N., and Hiraoka, Y.
\newblock Pore configuration landscape of granular crystallization.
\newblock~\emph{Nature Communications}, 8:15082, 2017.

\bibitem[Saggar et~al.(2018)]{saggar2018dynamical}
Saggar, M., Sporns, O., Gonzalez-Castillo, J., Bandettini, P.~A., Carlsson, G., Glover, G., and Reiss, A.~L.
\newblock Towards a new approach to reveal dynamical organization of the brain using topological data analysis.
\newblock~\emph{Nature Communications}, 9(1):1399, 2018.

\bibitem[Sizemore et~al.(2019)]{sizemore2019importance}
Sizemore, A.~E., Phillips-Cremins, J.~E., Ghrist, R., and Bassett, D.~S.
\newblock The importance of the whole: Topological data analysis for the network neuroscientist.
\newblock~\emph{Network Neuroscience}, 3(3):656--673, 2019.

\bibitem[Stone(1948)]{stone1948generalized}
Stone, M.~H.
\newblock The generalized Weierstrass approximation theorem.
\newblock~\emph{Mathematics Magazine}, 21:167--184, 237--254, 1948.

\bibitem[Sun et~al.(2009)]{sun2009hks}
Sun, J., Ovsjanikov, M., and Guibas, L.
\newblock A concise and provably informative multi-scale signature based on heat diffusion.
\newblock In~\emph{Computer Graphics Forum}, volume~28, pp. 1383--1392, 2009.

\bibitem[Tancik et~al.(2020)]{tancik2020fourier}
Tancik, M., Srinivasan, P.~P., Mildenhall, B., Fridovich-Keil, S., Raghavan, N., Singhal, U., Ramamoorthi, R., Barron, J.~T., and Ng, R.
\newblock Fourier features let networks learn high frequency functions in low dimensional domains.
\newblock In~\emph{Advances in Neural Information Processing Systems}, 2020.

\bibitem[The GUDHI Project(2015)]{gudhi2015}
The GUDHI Project.
\newblock~\emph{GUDHI User and Reference Manual}.
\newblock GUDHI Editorial Board, 2015.
\newblock \url{https://gudhi.inria.fr/}.

\bibitem[Townsend et~al.(2020)]{townsend2020}
Townsend, J., Micucci, C.~P., Hymel, J.~H., Maroulas, V., and Vogiatzis, K.~D.
\newblock Representation of molecular structures with persistent homology for machine learning applications in chemistry.
\newblock~\emph{Nature Communications}, 11:3230, 2020.

\bibitem[Umeda(2017)]{umeda2017betti}
Umeda, Y.
\newblock Time series classification via topological data analysis.
\newblock~\emph{Transactions of the Japanese Society for Artificial Intelligence}, 32(3):D-G72, 2017.

\bibitem[Vaswani et~al.(2017)]{vaswani2017attention}
Vaswani, A., Shazeer, N., Parmar, N., Uszkoreit, J., Jones, L., Gomez, A.~N., Kaiser, {\L}., and Polosukhin, I.
\newblock Attention is all you need.
\newblock In~\emph{Advances in Neural Information Processing Systems}, pp. 5998--6008, 2017.

\bibitem[Verma et~al.(2025)]{verma2025pipe}
Verma, Y., Souza, A.~H., and Garg, V.
\newblock Positional encoding meets persistent homology on graphs.
\newblock In~\emph{International Conference on Machine Learning}, 2025.

\bibitem[Wang et~al.(2024)]{multiset-transformer-2025}
Wang, M., Huang, Z., and Xu, J.
\newblock Multiset Transformer: Advancing Representation Learning in Persistence Diagrams.
\newblock~\emph{arXiv preprint arXiv:2411.14662}, 2024.

\bibitem[Wee \& Xia(2021)]{weexia2021}
Wee, J. and Xia, K.
\newblock Forman persistent {Ricci} curvature ({FPRC})-based machine learning models for protein--ligand binding affinity prediction.
\newblock~\emph{Briefings in Bioinformatics}, 22(6):bbab136, 2021.

\bibitem[Xia \& Wei(2014)]{xia2014persistent}
Xia, K. and Wei, G.-W.
\newblock Persistent homology analysis of protein structure, flexibility, and folding.
\newblock~\emph{International Journal for Numerical Methods in Biomedical Engineering}, 30(8):814--844, 2014.

\bibitem[Xiong et~al.(2020)]{xiong2020layer}
Xiong, R., Yang, Y., He, D., Zheng, K., Zheng, S., Xing, C., Zhang, H., Lan, Y., Wang, L., and Liu, T.-Y.
\newblock On layer normalization in the Transformer architecture.
\newblock In~\emph{International Conference on Machine Learning}, pp. 10524--10533, 2020.

\bibitem[Xu et~al.(2019)]{xu2019gin}
Xu, K., Hu, W., Leskovec, J., and Jegelka, S.
\newblock How powerful are graph neural networks?
\newblock In~\emph{International Conference on Learning Representations}, 2019.

\bibitem[Yoo et~al.(2016)]{yoo2016vineyard}
Yoo, J., Kim, E.~Y., Ahn, Y.~M., and Ye, J.~C.
\newblock Topological persistence vineyard for dynamic functional brain connectivity during resting and gaming stages.
\newblock~\emph{Journal of Neuroscience Methods}, 267:1--13, 2016.

\bibitem[You et~al.(2024)]{you2024phlp}
You, J., Heo, E., and Jung, J.-H.
\newblock {PHLP}: Sole persistent homology for link prediction --- interpretable feature extraction.
\newblock~\emph{arXiv preprint arXiv:2404.15225}, 2024.

\bibitem[Zaheer et~al.(2017)]{zaheer2017deepsets}
Zaheer, M., Kottur, S., Ravanbakhsh, S., Poczos, B., Salakhutdinov, R., and Smola, A.~J.
\newblock Deep Sets.
\newblock In~\emph{Advances in Neural Information Processing Systems}, volume~30, pp. 3391--3401, 2017.

\bibitem[Zhao et~al.(2021)]{zhao2021point}
Zhao, H., Jiang, L., Jia, J., Torr, P.~H., and Koltun, V.
\newblock Point transformer.
\newblock In~\emph{IEEE/CVF International Conference on Computer Vision}, pp. 16259--16268, 2021.
\end{thebibliography}
\end{document}